\documentclass[
aps,%
10pt,%
final,%
notitlepage,%
oneside,%
twocolumn,%
nobibnotes,%
nofootinbib,%
superscriptaddress,%
centertags,%
showkeys,%
amssymb,%
floatfix,%
]{revtex4} 

\usepackage[utf8]{inputenc}
\usepackage{amsfonts}
\usepackage{amsmath}
\usepackage{amsthm}
\usepackage{amscd}
\usepackage{graphicx}
\graphicspath{{/}{figures/}}
\DeclareGraphicsExtensions{.pdf,.png}
\usepackage[caption=false]{subfig}
\usepackage{epstopdf}
\usepackage{todonotes}
\usepackage{tabularx,booktabs}
\usepackage{enumitem}
\usepackage[pdftex,colorlinks=true]{hyperref}

\newcounter{mycomment}

\SetLabelAlign{parleft}{\parbox[t]{\labelwidth}{\raggedleft\textbf{#1}}}
\setlist[description]{align=parleft,labelindent=0.0ex,labelwidth=7ex}

\newtheorem{Definition}{Definition}
\newtheorem{Theorem}{Theorem}
\DeclareMathOperator{\arccosh}{arccosh}
\DeclareMathOperator{\Ima}{Im}
\newcommand{\m}[1]{\ensuremath{\mathrm {#1}}}
\newcommand*{\dd}{\ensuremath{\mathrm{d}}}
\newcommand*{\ee}{\ensuremath{\mathrm{e}}}
\newcommand*{\ii}{\ensuremath{\mathrm{i}}}

\newcommand*{\ICHAIN}{\ensuremath{(\mathrm{ITD}-\mathrm{IMF})\mathrm{chain}_{1}}}

\begin{document}

\title{Support Spinor Machine}

\author{Kabin Kanjamapornkul}
\email{kabinsky@hotmail.com}
\affiliation{Department of Computer Engineering, Faculty of Engineering, Chulalongkorn University, Pathumwan, Bangkok 10330, Thailand}

\author{Richard Pinčák}
\email{pincak@saske.sk}
\affiliation{Institute of Experimental Physics, Slovak Academy of Sciences, Watsonova 47, 043 53 Košice, Slovak Republic}
\affiliation{Bogoliubov Laboratory of Theoretical Physics, Joint Institute for Nuclear Research, 141980 Dubna, Moscow Region, Russia}

\author{Sanphet Chunithpaisan}
\email{sanphet.c@chula.ac.th}
\affiliation{Department of Computer Engineering, Faculty of Engineering, Chulalongkorn University, Pathumwan, Bangkok 10330, Thailand}

\author{Erik Bartoš}%
\email{erik.bartos@savba.sk}
\affiliation{Institute of Physics, Slovak Academy of Sciences, Dúbravská cesta 9, 845 11 Bratislava, Slovak Republic
}%

\begin{abstract}
We generalize a support vector machine to a support spinor machine by using the mathematical structure of  wedge product over vector machine in order to extend field from vector field to spinor field. The separated hyperplane is extended to Kolmogorov space in time series data which allow us to extend a structure of support vector machine to a support tensor machine and a support tensor machine moduli space. Our performance test on support spinor machine is done over one class classification of end point in physiology state of time series data after empirical mode analysis and compared with support vector machine test. We implement algorithm of support spinor machine by using Holo-Hilbert amplitude modulation for fully nonlinear and nonstationary time series data analysis.
\end{abstract}

\keywords{support spinor machine, support vector machine, time series, moduli state space, holo-Hilbert spectrum}

\maketitle

\section{Introduction}\label{sec:intro}

In present time, most researchers in econometrics \cite{econ} are turning their interests to higher mathematical invariant property so called the Gaussian curvature, especially in a hypothesis test \cite{curvature}. On the other side there exists an active research \cite{pincak21,pincak20} in a machine learning and a support vector machine (SVM) \cite{vapnik} applying the usage of modern mathematical structure behind Riemannian curvature so called curvature tensor and spinor field \cite{cw}. The differential geometry and cohomology theory approach \cite{cohomo} was introduced for the study of a geometry of an arbitrage opportunity as a new quantity in financial market in the modern econometric theory.

Recent studies of a deep learning in convolution network \cite{deep,deep5} and on twin SVM \cite{twin,twin5} are still based on statistical learning theory over Euclidean plane with Hausdorff separation criterion $T_{2}$ over induced underlying topological structure of feature space of classification. Naturally, one can extend a scalar field to a Killing vector field by using the Lie derivative of a tensor field and finally to spinor field and to induce an existences of mathematical structure of Clifford algebras over support spinor machine and artificial neuron network like a Basic Clifford neuron (BCN) \cite{spinor_neuron2,clifford_ann2}. Therefore, theoretically, SVM can be extended to a support tensor machine (STM) \cite{stm,stm2} by using an algebraic operation of free product over equivalent class of quadratic form. The spinor field in the form of Clifford algebra have been directly used \cite{spinor_neuron1} in a neuroscience and cognitive science to improve mathematics in classical model of support vector machine up to so called Clifford fuzzy SVM \cite{clifford_fuzzy}, but with a lack of empirical analysis with financial time series data. 	The financial time series data by nature \cite{empirical} are nonlinear and nonstationary time series data and it is very difficult to classify the future direction up and down for them. Because of only simulation works based on both BCN and Clifford fuzzy SVM have been done, one needs to apply a generalization of a vector field to high dimensional mathematical object so called spinor field with real empirical result on financial time series in order to classify their nonlinear and nonstationary behavior. In algebraic geometry a classification space is studied \cite{teichmuller} by using the invariant property of curvature over Teichmüller space. It is selected as a mathematical object to classify a feature space in the space of time series data. This fact gives us a new theory of quantum probability over hyperbolic number \cite{hyper_number}, and an extradimension approach \cite{pincak1,pincak2,pincak3,pincak4} for time series data prediction.

We observe that there exists some algebraic topological defect in the mathematical structure of machine learning for support vector classification. The separation in support vector machine \cite{Vapnik2} performs the classification based on Euclidean plane. This plane can not be used for the classification of some types of input data in the entanglement state with topological Hopf fibration space underlying the input feature space of time series data in non Euclidean plane. Typically, SVM is based on $T_{2}$ topological space in which distinct points have disjoint neighborhoods. This criterion is too restrictive for SVM and it does not match with nondisjoint feature set of the classification. Many researchers in STM try to solve this classification problem of SVM under $T_{2}$-separation by introducing the tensor product over feature space and kernel of SVM in which they do not know that in fact the space of time series data is in $T_{0}$-separation space. $T_{0}$ space implies $T_{2}$ but $T_{2}$ property can not always imply $T_{0}$-separation in SVM. In order to solve this problem, we introduce a support spinor machine (SSM) and a support Dirac machine (SDM) as a generalization of SVM based on mathematical structure of connection and Riemannian curvature over principle bundle of Kolmogorov space in time series data. We can classify input data with  $T_{0}$-separation axiom of spinor field in time series data over non-Euclidean plane of the classification instead of Haussdorf $T_{2}$-separation axiom in SVM over Euclidean plane. The support spinor machine is a new mathematical object arising from the algebraic topological and machine learning approaches on the level of the invariant property over topological space underlying the observed data, so called moduli space in time series data. SVM poses the intrinsic property of the classification allowing higher dimension of feature space for the classification embedded in the complex projective space with oriented and non-oriented states instead of using pointed space embedded in the real number system.

In SSM we are based on the concept of adaptive data analysis method of Hilbert-Huang transformation \cite{hht} and Holo-Hilbert amplitude modulation \cite{holo} suitable for the empirical analysis of financial time series data induced from adaptive behavior of traders \cite{adaptive}. We can not defined explicit statistical formulae for learning algorithm of SSM. Instead we use data learning themself from the sifting process of moduli state space model. It is implied that we do empirical analysis in more flexible form than BCN and Clifford fuzzy SVM in which it is based on explicit form of data simulation of Clifford algebras. The contribution of SSM over BCN and Clifford fuzzy SVM  lays in fact, that in SSM we use invariant property of toplological space and also second cohomology group of spinor field to analyze data. We use moduli state space to visualized data in SSM which are generalization of data points in BCN and Clifford fuzzy SVM. The separation hyperplane in SSM is based not only on algebraic property of spinor field but also on geometric property of loop space in time series data in which we use invariant property of plane and duality map between second cohomology group as spinor field. Although BCN and Clifford fuzzy SVM also can be applied to real analysis to nonlinear and nonstationary data like SSM if we need to compare the performance to real data set between them and SSM.

In this work, we give the precise definition of moduli state space model in financial time series data and we prove the existence of support spinor field and its application with algebraic topological approach. We verify the existence of spinor field by using mathematical invariant property of second cohomology group and moduli space. To guarantee the usefulness of new SSM model, the empirical analysis have been done over financial time series data. We compare the result of performance of one day ahead directional prediction with support vector machine. We find the equation of moduli state space model over support Dirac machine. In the equation, we have a solution of weight of SSM as a coupling tensor field between two sides of market in the behavior of trader in hidden state space $X_{t}([A_{1}],[A_{2}],[A_{3}])$ ($[A_i]$ are defined in \cite{cohomo}) and in a physiology of time series data. This equation of moduli state space model is an equivalent class over moduli state space $C_{n}(Y_{t}/X_{t})$ between observed state space, $Y_{t}([s_{1}],[s_{2}],[s_{3}],[s_{4}])$ (the definition of physiology state, $[s_{i}]$ can be found in \cite{cw}) and hidden state space, $X_{t}([A_{1}],[A_{2}],[A_{3}])$.

Based on the proposed mathematical motives we state that SSM superiors SVM with nonstationary data analysis \cite{review}. Basically, SVM are used to classifying data with constant time. For time series data, there exists support vector regression machines \cite{srm} based on SVM concept and cannot be used for general time series data.

The advantage of SSM over SVM is that SSM can be useful for classifying future direction of nonlinear \cite{nonlinear} and nonstationary \cite{nontstationary,nontstationary2} time series data, especially financial time series data. SVM can be used only for nonlinear but stationary time series data based on linear algebras of vector space for support vector classification. But SSM are based on a spinor field in which can be linear and nonlinear vector space and Holo-Hilbert spectral analysis \cite{holo} in which can be done fully nonstationary time series analysis. Based on criterion of nature of input data, SSM can be used in more general conditions than SVM. 
 
The paper is organized as follows, in Section \ref{sec:ssm} we prove the existence of the support spinor machine and the moduli state space model. We explain the application of SSM algorithm in the form of classification of point over complex projective plane. We explain the market cocycle in the form of coupling between the tensor field in time series data. We build a new equation for financial market in the form of Killing vector field over connection and we derive the explicit formula of an equivalent class of weight in SSM. In Section \ref{sec:results}, we apply SSM to the directional prediction over six different types of time series data. We measure the performance of the directional prediction in out of sample data. We compare the performance of SSM with SVM and the empirical result of Holo-Hilbert transform amplitude modulation (AM) mode with second layer. In Section \ref{sec:conclusions}, we give the discussion and conclusion for our new model of SSM and the plans for the future work.

\section{Support Spinor Machine}\label{sec:ssm}

\subsection{Moduli State Space Model}

\begin{figure*}[!t]
	\centering
	\includegraphics[width=.8\linewidth]{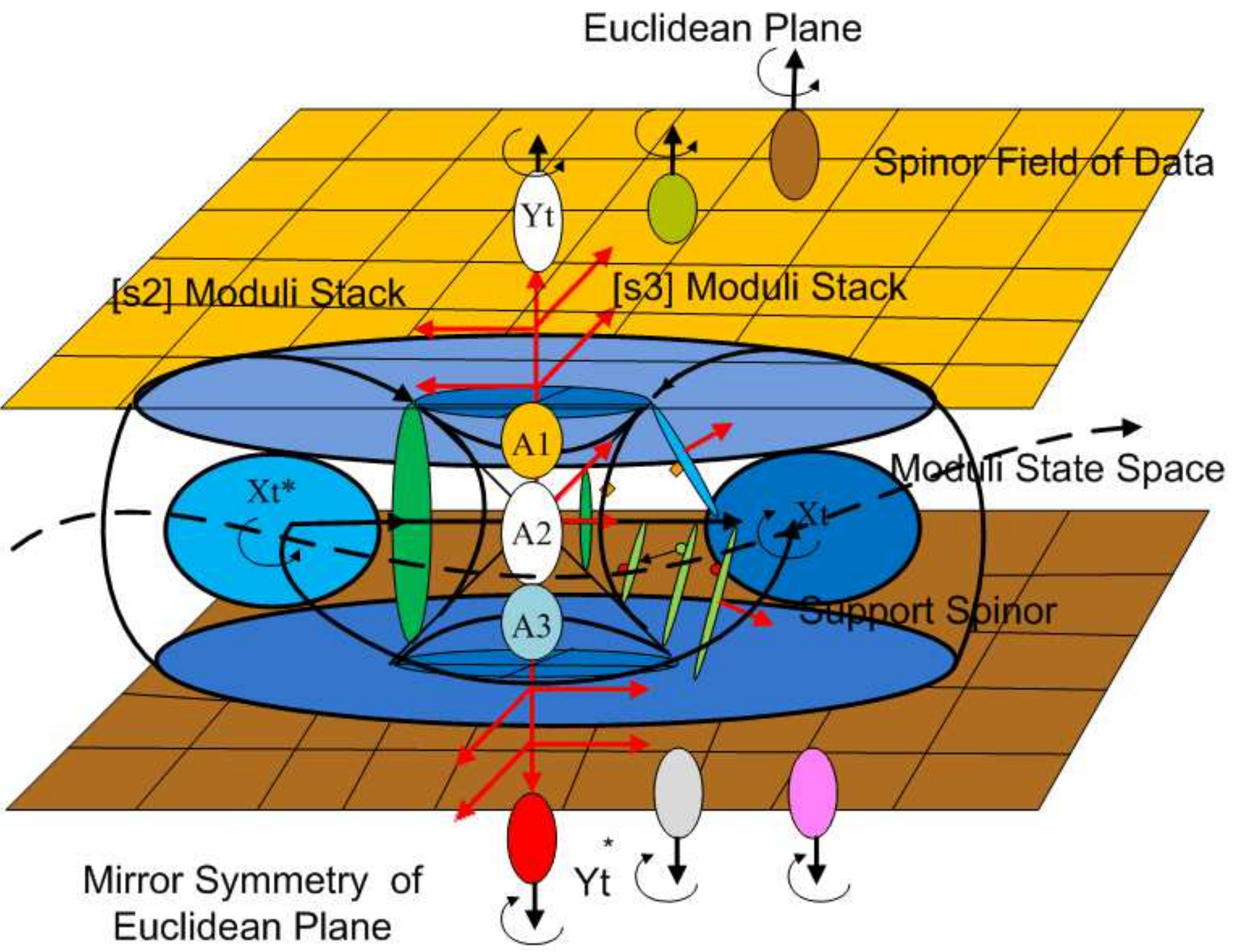}
	\caption{Picture shows the moduli state space model with classical support vector. We use the error of prediction to build a cone of separated hyperplane in quantization field of spinor machine. It is supposed to be Teichmüller space in time series data.\label{ssm}}
\end{figure*}

In $T_{0}$-separation criterion of time series data, we obtain a loop space in time series data. In this section we want to introduce standard classification of topological space by using genus of Riemann surface so called Teichmüller space in time series data. Each data of classification we associate with a genus of Riemann surface of data in the form of moduli stack of time series data analogue with moduli stack $\mathcal{M}_{g}$ in algebraic geometry (Fig. \ref{ssm}).  We use Riemann surface $S^{2}\sim \mathbb{C}P^{1}$, where $\sim$ is the homotopy equivalent, as the surface of principle bundle for embedding data of classification with genus zeros, non-Eculidean plane in support spinor machine. The algorithm for embedding data to Riemann surface with genus $g$ as a toric model of moduli space $\mathcal{M}_{g}$ of data so called Teichmüller map of time series data $x_{t}\in X_{t}=\mathcal{M}_{g}$ is,
\begin{equation*}
\gamma:[0,1]\rightarrow \mathcal{M}_{g},\quad \gamma(0)=x_{t},\quad \gamma(1)=\m{Id}.
\end{equation*}
Given a time series data as a sequence of observation data $x_{1}\rightarrow x_{2}\rightarrow \cdots \rightarrow x_{n}$, we induce a sequence of embedded Teichmüller spaces in time series data as following
\begin{equation}
\begin{CD}
x_{1} @>>> x_{2} @>>> \cdots@>>>x_{n} \\
@VVV @VVV @VVV @VVV\\
\mathcal{M}_{g=1} @>{\text{d}}>> \mathcal{M}_{g=2} @>{\text{d}}  >>  \cdots@>{\text{d}}>> \mathcal{M}_{g=n}
\end{CD}
\end{equation}
where $\mathcal{M}_{g=n}$ is the moduli state space of time series data which is isomorphic to Riemann surface of data with $n$-hole (genus $n$).

\begin{Definition}
Let $X_{t}$ be a Kolmogorov space in time series data of input financial time series data for the classification, $X_{t}=X_{t}([A_{i}])$, $i=1,\dots 3$, where $[A_{i}]$ is an equivalent class of the behavior of trader. We have $2^{3}=8$  hidden market states as all possibilities (discrete topology of categories of coordinate of $[A_{i}]$) of transition between all hidden states.
\end{Definition}

\begin{Definition}
Let $Y_{t}$ be a Kolmogorov space in time series data of input financial time series data for the classification, $Y_{t}=Y_{t}([s_{i}])$, $i=1,\dots,4$, where $[s_i]$ is a physiology of time series data in observation or classification space.
\end{Definition}

\begin{Definition}
Let $C_{n}(X_{t})$ be a chain complex of state space and $C_{n}(Y_{t})$ be a $n$-chain complex of observation space. We have a differential of chain $C_{n}(Y_{t})$ and cochain $C^{n}(Y_{t})$ between state and observation complex. We define a relative complex $C_{n}(Y_{t}/X_{t}):=C_{n}(Y_{t})/C_{n}(X_{t})$. The cohomology group for moduli state space model with equivalent class of moduli regression coefficient $[\beta]$ is defined for $n=0,1,2,\dots$ by
\begin{gather*}
[\beta_{n}]\in  H_{n}(Y_{t}/X_{t})=\ker \partial_{n-1}C_{n-1}(Y_{t})/\Ima \partial_{n}C_{n}(Y_{t}).
\end{gather*}
\end{Definition}
%

\begin{Definition}
A equivalent path $[\alpha_{n}]$ in support spinor machine (SSM) is an equivalent class of degree of map between homotopy class of Techm\"{u}ller space in time series data $\mathcal{M}_{g}$ with genus $g$ to classifying space $Y_{t}=\mathbb{Z}_{2}$
\begin{equation}
[\gamma_{i}]\in[\mathcal{M}_{g} , \mathbb{Z}_{2}]:=[\mathcal{M}_{g},Y_{t}]
\end{equation}
with $[\gamma_{i}]=[(e^{i\theta},e^{i\beta})]\mapsto (\alpha^{\ast},\beta^{\ast})\mapsto \mathbb{Z}_{2}. $
\end{Definition}

\begin{Definition}
A Holo-Hilbert functor in time series data is an equivalence between an functor map of homotopy class in higher dimension of $n$-sphere defined by
\begin{multline}
\m{Holo}:[S^{0},-]\rightarrow [S^{1},-]\\ 
\rightarrow [S^{2},-]\rightarrow \cdots [S^{n},-]\rightarrow   \cdots 
\end{multline}
with
\begin{multline*}
\m{Holo}(X_{t})=[S^{0},X_{t}]\rightarrow [S^{1},X_{t}]\\\rightarrow [S^{2},X_{t}]\rightarrow \cdots [S^{n},X_{t}]\rightarrow  \cdots
\end{multline*}
\end{Definition}

\begin{Definition}
A loop space in time series data is an observation state space of a physiology of time series data defined by $Y_{t}:= \pi_{1}(  \mathbb{R}^{n+1}-X_{t} )$. We define a spinor field in time series of observation data by equivalent class of duality map between  $y_{t}\in \pi_{1}(\mathbb{R}^{n+1}-\{x_{t}\}):=Y_{t}$, $x_{t}\in X_{t}$ be a classifying observation data in observation state space, $y_{t}\in Y_{t}$
\begin{multline}
[w]:\pi_{1}(  \mathbb{R}^{n+1}-X_{t} ):=Y_{t}\rightarrow H_{1}(X_{t})\\ \rightarrow  H_{2}(X_{t})\rightarrow H^{2}(X_{t}).
\end{multline}
\end{Definition}

\begin{Definition}
We define an equilibrium state in moduli state space model by an isotopy group under group operation of $G=PSL(2,\mathbb{C})$,
\begin{equation}
G\times X_{t}\rightarrow X_{t}
\end{equation}
Let $Y_{t}$ be a time series of measurement financial data induced from coupling between equivalent class of behavior of trader $[A_{i}]$ $(i=1,2,3)$. Consider $y_{t}\in Y_{t}$ with $y_{t}=G([A_{1},A_{2},A_{3}]).$ The moduli state space model is a moduli of classical  ARIMA model in time series analysis with coeficient as parameter in quantization state of spinor field $ a,b,c,d,\alpha \in \mathbb{H}$ into  a modified diophantine equation \cite{diophan} (just borrowing the notation with different algebraic operation in new definition) as following
\begin{align}
a x_{t} &\equiv 1 \quad mod \quad [s_{1}]
\\
b x_{t} &\equiv i \quad  mod \quad [s_{2}]
\\
c x_{t} &\equiv -1 \quad  mod \quad[s_{3}]
\\
d x_{t} &\equiv -i \quad  mod \quad[s_{4}]
\\
\alpha y_{t} &\equiv x_{t}([s_{i}],A[i])  \quad mod \quad [A_{i}].
\end{align}
We have
\begin{align}
a x_{t} - [s_{1}] x_{t}^{\ast} &= 1  \\
b x_{t} - [s_{2}] x_{t}^{\ast} &= i  \\
c x_{t} - [s_{3}] x_{t}^{\ast} &= 1  \\
d x_{t} - [s_{4}] x_{t}^{\ast} &= -i  
\end{align}
where $x_{t}^{\ast}$ is a hidden state of time series data.
\end{Definition}
 
\subsection{Proof of $T_{0}$-separation criterion for SSM}

The source of SVM is a second cohomology group of the classification of a plane by using a loop space in time series data. Let $X_{t}$, $Y_{t}$ be a topological space of state and observation time series data with $Y_{t}=\pi_{1}(\mathbb{R}^{n+1}-X_{t})$. Let $x_{1},x_{2}\in X_{t}$, $[-1]\in\pi_{1}(\mathbb{R}^{2}-\{x_{1}\})$ and $[1]\in\pi_{1}(\mathbb{R}^{2}-\{x_{2}\})$
\begin{equation}
\pi_{1}(\mathbb{R}^{2}-X_{t})\rightarrow  H_{1}(X_{t}) \rightarrow  H_{2}(X_{t}).
\end{equation}

\begin{Definition}
$<X_{t},\mathcal{O}>$ is a $T_{0}$-space if and only if $x,y\in S$ with $y\neq x$ implies that $\exists U,V\in \tau$ with $x\in U,y\in S-U$ and $y\in V,x\in S-V$. $T_{0}$-spaces are also called Kolmogorov spaces.
\end{Definition}

\begin{Definition}
$<X_{t},\mathcal{O}>$ is a $T_{2}$-space if and only if $x,y\in S$ with $y\neq x$ implies that $\exists U,V\in \tau$ with $x\in U,y\in V$ and $U\cap V=\phi$. $T_{2}$-spaces are also called Hausdorff spaces.
\end{Definition}

\begin{figure}[!tb]
	\centering
	\includegraphics[width=\linewidth]{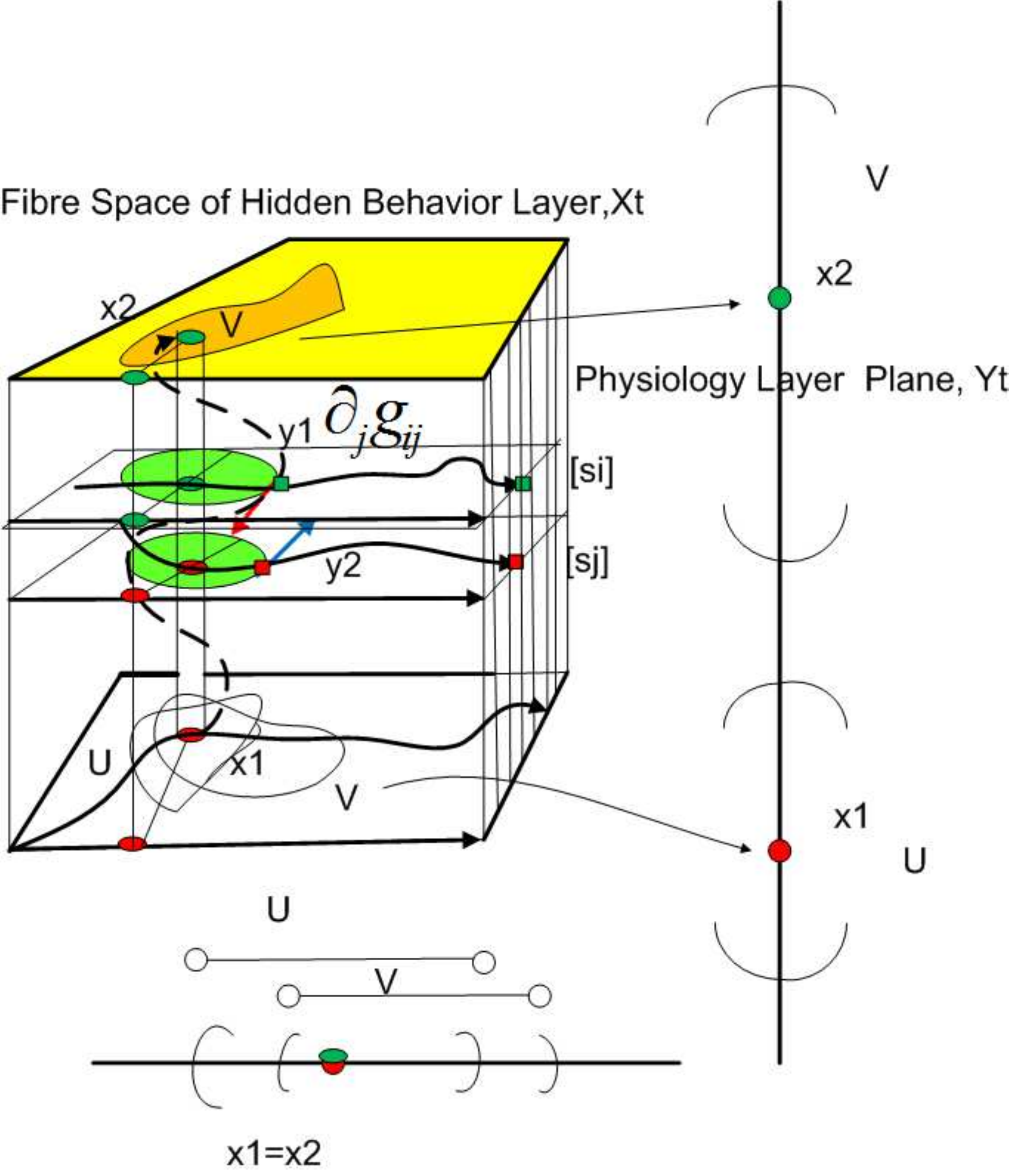}
	\caption{The projection of similar input data to real number line. We lift the path of data to fibre space, it is the source of curvature in time series data. The curvature blend a real line into curve with operator of Riemannian tensor field $\partial_{k} g_{ij}$. The side of cute in this plot is moduli state space model of physiology of time series data. This space is an observation space of time series data. Its source is coming from the coupling of behavior of trader in the hidden state space $X_{t}([A_{i}])$ layers. In SVM, there exists an algebraic defect over Euclidean plane of separation. This situation happens when input feature space contains the same data with $x_{1}=x_{2}\in U,V\subset \mathbb{R}^{n},$ but $(x_{1},y_{1})\neq (x_{2},y_{2})$. In SVM, we use $T_{0}$-separation axiom to solve this problem with the lift path of $x_{1}$, $x_{2}$ to the fibre space with homotopy path. The separation is a diffeomorphic map in the tangent of manifold of Kolmogorov space in time series data. \label{proof}}
\end{figure}

\begin{Theorem}
Let $(x_{1},y_{1})$ and $(x_{2},y_{2})$ be input space and classifying space for a support vector machine in the Euclidean plane such that $x_{1}=x_{2}$  and $y_{1}\neq y_{2}, x_{1},x_{2}\in \mathbb{R}^{n}, y_{1},y_{2}\in \mathbb{Z}_{2}$. We can separate a non-separation data in $T_{2}$-separation with a disjoint neighborhood in a moduli state space in time series data with $T_{0}$-separation of SSM.
\end{Theorem}
\begin{proof}
Let $x_{1},x_{2}\in \mathbb{R}^{n}$, with $x_{1}=x_{2}$ and $y_{1}\neq y_{2}\in  \mathbb{Z}_{2}=\{ -1,1 \}$, if $y_{1}=1$ we imply $y_{2}=-1$. For SVM, there exists some defect on algebraic construction over real line on classification of input data $(x_{i},y_{i})$, which can not be used when two input data have similar value with different output of classification $y_{i}\in \{+1,-1\}$. We suppose that the training data of $n$-samples are recorded and visualized as embedded pointed space $(x_{1},y_{1}),\cdots (x_{n},y_{n})$, $x\in \mathbb{R}^{n+1}$, $y\in \{+1,-1\}$ in the Euclidean space with hyperplane decision function
\begin{align}
<w,x>-b\geq +\triangle, \quad y=+1,\\
<w,x>-b\leq -\triangle, \quad y=-1,
\end{align}
where $w$ is normal vector to the hyperplane. In the Euclidean plane over real number field, there exists the case when the input space $x_{i}\in \mathbb{R}$ is not compact and complete space, e.~g., the space contains only two disjoint similar points with points defined to infinity $\ast=\{\infty\}$. Therefore one cannot use the compact equation because the value of classification ground field has the same algebraic structure with inner product $<w,x>$, $y\in \mathbb{R}$. This problem is well known in the functional analysis and it is so called ill-posed problem or regularization problem. Therefore we get the equation
\begin{equation}
y<w,x>-b\leq \triangle,
\end{equation}
which is not hold for SVM in the cases of ill-posed problem over classification data.
  
The spinor field of time series data model is a hidden stats of classical time series model throughout a quantization state of time series data. The quantization state of time series data can be explained by the concept of orientation and non-orientation state in time series data in which it is deep connected to the concept of tanglement and entanglement states in the quantum mechanics, in some situations we can resolve this problem over principle bundle theory.

Let us consider a neighborhood of classification data over input space $X=U\cup V$, with $U\cap V=\phi$. Let $U_{\epsilon}(x\in \mathbb{R}^{n};x_{i})$ defined by $U_{\epsilon}(x;x_{i})=\{x| ||x-x_{i}||<\epsilon  \}.$ Choose a radius ball $\epsilon$ such that  $wx-b<\epsilon <1$ then set $\epsilon'=\frac{\epsilon}{w}$ such that $x-\frac{b}{w}<\frac{\epsilon}{w}<1$, set $x_{i}=\frac{b}{w}$. We have $U_{\epsilon}(x;x_{i})=\{x| ||x-x_{i}||<\epsilon<1  \}$. Let $x_{1}\in U_{\epsilon}(x;x_{i})=:U$. It implies that $x_{2}\in U_{\epsilon}(x;x_{i}):=V\subset \mathbb{R}^{n}$. Therefore $U\cap V\neq \phi$. It is in the contradiction with the assumption above. Therefore, we can not use $T_{2}$-separation in the case of such algebraic defect over SVM. (We want to proof that in SSM the covering spaces of $X_{1}$ and $X_{2}$ are separated with $T_{0}$-separation axiom.)
\end{proof}

Let $x_{1},x_{2} \in \mathbb{R}^{n}$. It is enough to perform the proof for the case $n=2$. Let $T^{2}=\mathbb{R}^{2}/\mathbb{Z}$ be a torus with genus $g=1$. If we consider $\mathbb{R}^{2}-\{x_{1}\}$, this space is homotopy equivalent to $S^{1}$. We use the fact that $S^{2}=S^{1}\vee S^{1}\sim \mathbb{C}P^{1}$ where $\vee$ is a smash product in the algebraic topology and $\sim$ is a homotopy equivalence operation, the complex projective space $\mathbb{C}P^{1}$ is a principle bundle. Let define a lift path from $\mathbb{R}^{2}\rightarrow T^{2}$.

\begin{Theorem}
SVM cannot use $T_{2}$-separation axiom to separate similar values of input time series data over the Euclidean plane. We will prove that if the input feature space is extended to the non-Euclidean plane of a equivalent class of degree of map in the loop space in time series data then we can use $T_{0}$-separation axiom to separate the non-separated similar values of input data in SVM. The algebraic defect of SVM will be repaired by SSM with input feature space in the form of a equivalent class of degree of loop space in time series data $x_{1}:= [\theta_{1}]\in \pi_{1}(X_{t},x_{1})$, $x_{2}\:= [\theta_{2}]\in \pi_{1}(X_{t},x_{2})$ with $[\theta_{1}]=[\theta_{2}]$ but $\theta_{1}\neq \theta_{2}$.
\end{Theorem}

\begin{proof}
Let $(x_{i},y_{i})$ be a feature space of SSM with $x_{1}\in U\subset X_{t}$, $x_{2}\in V\subset X_{t}$ with $U\cap V=\phi$. Let $y_{1}=[e^{2i\pi}]=[1]\in \pi_{1}(\mathbb{R}^{n+1}-U)=H^{1}(U)$, $y_{2}=[e^{i\pi}]=[-1]\in \pi_{1}(\mathbb{R}^{n+1}-V)=H^{1}(V)$, $n=1$. We assume that the input $\{x_{i}\}$ is a pointed space. The point $x_{i}$ is contractible from compact feature space of data $U,V\sim \ast$. Let a deform plane of the classification for SSM be a Kolmogorov space in time series data $Y_{t}$. When we delete one point $x_{i}$ from the plane of classification $y_{i}\in \mathbb{R}^{n+1}$ we induce a vector field on unit cycle of the classification space $y_{i}\in S^{n}$ around input data $x_{i}$ since $\mathbb{R}^{n+1}-\{x_{i}\}\sim S^{n}$.

Let $x_{1}=x_{2}$ with $(x_{1},y_{1})\neq (x_{2},y_{2})\in U\cap V$. We want to prove that $x_{1}\in U$ and $x_{2}\not\in V$ or $x_{1}\not\in U$ and $x_{2}\in V$ over the fibre space by using the homotopy lift path (see Fig.~\ref{proof} for more details). We want to remark that if $U,V\subset  \mathbb{R}^{n}$ or the Euclidean flat plane, we cannot separate $x_{1}$ from $x_{2}$ as we have demonstrated in the previous section. But if $U$, $V$ are subset of the non-Euclidean plane we can separate them by using homotopy to deform the Kolmogorov space of time series data. The defintion of a homotopy map in this prove is used to glue the deleted points from the  classification space to Riemann sphere $S^{2}\sim\mathbb{C}P^{1}$  
\begin{equation}
H:\mathbb{R}^{3}-\{x_{i}\}\times I \rightarrow S^{2}.
\end{equation}
In the second part of proof we let $x_{1}\in U, x_{2}\in V$ with $X_{t}=U\coprod_{\alpha} V$ and the homotopy class above is defined by 
\begin{gather}
[\alpha]:U\coprod V\rightarrow S^{1},\\ \nonumber (x_{1},0)\mapsto [-1]=e^{i\pi}\in S^{1},(x_{2},1)\mapsto [1]=e^{i2\pi}\in S^{1}. 
\end{gather}
We let $\theta_{1}=\pi$, $\theta_{2}=2\pi$, it is clear that $\theta_{1}\neq \theta_{2}$. There exists the optimal degree $\theta^{\ast}$ such that the line passing $\theta^{\ast}$ cuts unit cycle and separates $\theta_{1}=\pi$ and $\theta_{2}=2\pi$ from each other. We have chosen the separate plane over the complex projective plane with  $[\theta^{\ast}]=\frac{3\pi}{4}$, so we have  $\theta_{1}<\theta^{\ast} < \theta_{2}$. 
\end{proof}   

The source of the spinor field in time series data induces from the behavior of trader for the expectation on next period of the physiology of time series data state $[s_{i}]_{i+1}:=[s^{\ast}_{i}]$ in financial time series data. There exists many possibilities in the expected futures price from various behaviors of trader. We can classify them into four types of tensor fields under moduli state space model in the next section. The source of SSM in the classification of state in the financial market induces from coupling between these behavior fields of trader in the financial market. We explain in detail with the definition of action of these fields to space of financial market as market cycle and cocycle quantity. These quantities are equivalent to the genus of our new model of Teichmüller space in time series data in the previous subsection.

\subsection{Financial Market Cocycle}

The element $g\in G$ induces a Lie algebras \cite{lie_algebra} $g_{ij}$ in $T_{x_{t}}G$ over the tangent of spin manifold of time series data with jacobian flow $\frac{dg_{ij}}{dt}$, where $t\in [0,1]$ of phase space of transition over homotopy class \cite{massey} of space of time series data. We define so called cocycle property of orbifold in time series data with its dual $g_{k}^{l}=g_{ki}g_{im}g_{m}^{l}$. We call $g_{ij}$ a covariance tensor field in financial time series data, $g^{ij}$ is a contravariance tensor field in financial data and $g_{k}^{l}$ a natural map between covariance and contravariance functor in time series data. We denote the market state in the physiology layer of time series data $x_{t}$ as $[s_{i}]$ and in the main layer of time series data as hidden market state $[A_{i}]$. The expectation states of the prediction we denote as $[s_{i}^{\ast}]$ and $[A_{i}^{\ast}]$. In this work we use a symbol of a homotopy path $[\beta_{t+1}]$ as a equivalent class of path over behavior of trader $[A_{i}^{\ast}]:=[\beta_{t+1}]$,
and a symbol $[\theta_{t+1}]$ as a equivalent class of path over behavior of trader $[s_{i}^{\ast}]:=[\theta_{t+1}]$.
The jacobian is defined as the market cocycle between the hidden coordinate transformation of behavior of trader to the coordinate of the physiology of time series data in outer layer of observation state space $Y_{t}$
\begin{equation}
g_{ij}=\Bigg(\frac{\partial [A_{j}]}{\partial [s_{j}]}\Bigg)_{ij}.
\end{equation}
The scalar product over the tangent of manifold of space of time series data is define over principle bundle $\mathbb{C}P^{1}$
\begin{equation}
<\beta_{t+1},\theta_{t+1}>\simeq \Big[\frac{\beta_{t+1}}{\theta_{t+1}},1\Big]\in \mathbb{C}P^{1}.
\end{equation}

\begin{Definition}
Let define four types of tensor fields in time series data with extradimensions in fibre space with coupling between vector and covector field in moduli state space in time series data $x_{t}\in X_{t}$, $y_{t}\in Y_{t}$, $[\theta_{t+1}]\in [X_{t+1},S^{1}]$, $[\beta_{t+1}]\in [Y_{t+1},S^{1}]$. It is an expectation field of behavior of trader with four component fields
\begin{gather}
\partial_{j}:=\frac{\partial}{\partial [\theta_{t+1}]},\quad \partial_{k}:=\frac{\partial}{\partial x_{t+1}}, \nonumber\\ \partial_{l}:=\frac{\partial}{\partial y_{t+1}},\quad \partial_{m}:=\frac{\partial}{\partial [\beta_{t+1}]}
\end{gather}
\end{Definition}

\begin{Definition}
Let $g_{ij}$ be a market cycle and $g^{ij}$ be a market cocycle over the moduli state space of Kolmogorov space $X_{t}$, $Y_{t}$. The cocycle is an expected futures cycle. Such model is a market communication model between supply and demand side of the physiology layer of time series data and the behavior of trader in hidden layer of time series data.

\begin{figure}[!t]
\centering
\includegraphics[width=.8\linewidth]{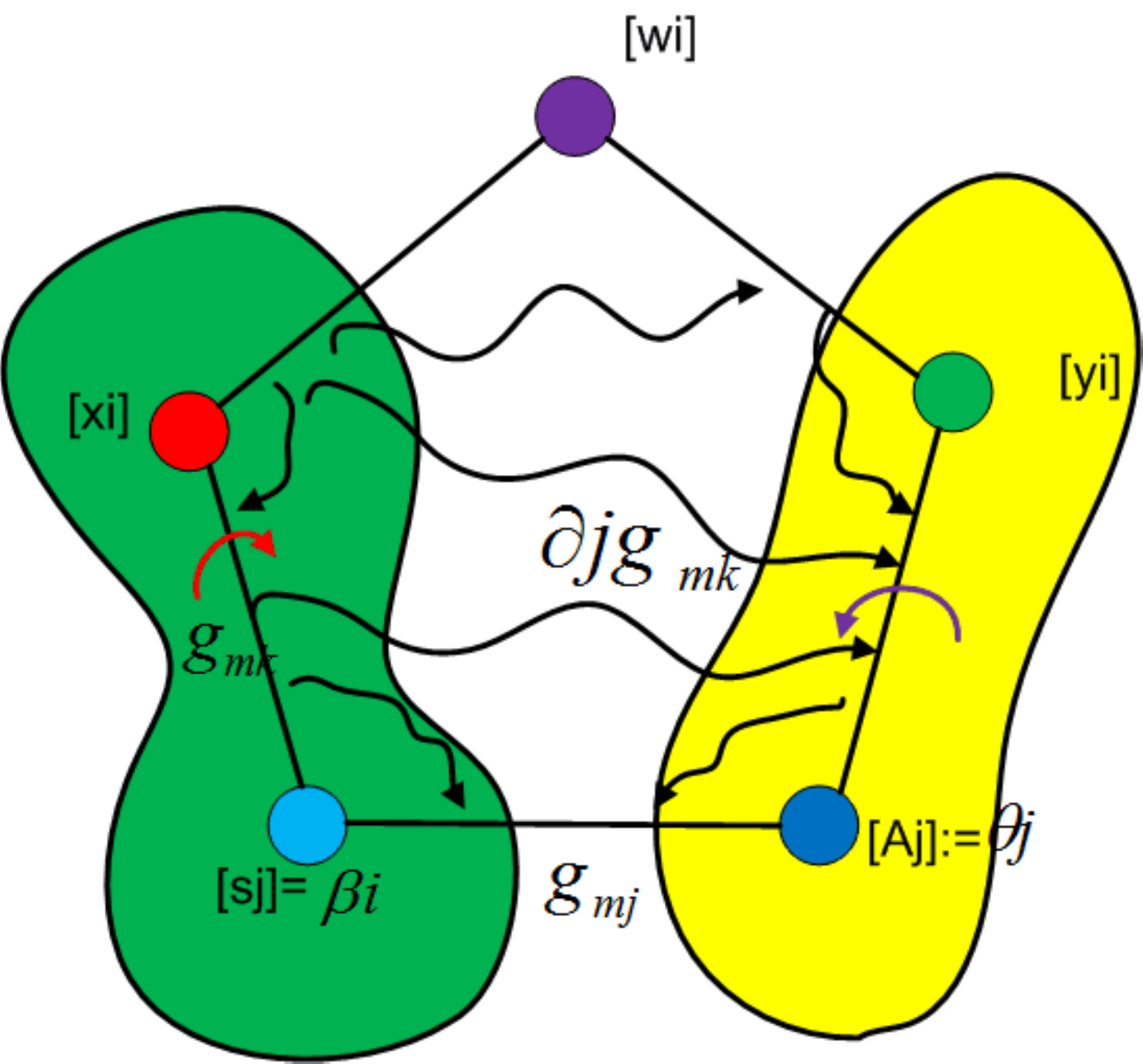}
\caption{The picture shows the example of the definition of four types of tensor field between the physiology layer of time series data $Y_{t}=Y_{t}([s_{1}], [s_{2}], [s_{3}], [s_{4}])$ and main layer in time series data so called hidden behavior of trader $X_{t}([A_{1}, A_{2}, A_{3}])$. The connection is the strength of a tensor field, i.~e., the map between sides in the diagram. \label{tensor}}
\end{figure}
  
We define each component of tensor field in time series data by jacobian of coordinate transformation $g_{ij}$ with explicitly four components of jacobian flow associated with four types of market cycles as the followings (see Fig.~\ref{tensor} for details of the definition)
\begin{itemize}
\item \emph{Type I.} Let $g_{kl} $ is a cycle associate with scalar product of $<x_{t+1},y_{t+1}>=\sum_{kl}g_{kl}x_{k}y_{l}:=\sum_{kl}g_{kl}x_{t+1,k}y_{t+1,l}$, $g^{kl} $ is a cocycle associate with scalar product of $<x_{t+1}^{\ast},y_{t+1}^{\ast}>=\sum_{kl}g^{kl}x_{k}^{\ast}y_{l}^{\ast}:=\sum_{kl}g^{kl}x_{t+1,k}^{\ast}y_{t+1,l}^{\ast}$.
\item \emph{Type II.} Let $g_{ml} $ is a cycle associate with scalar product of $<[\beta]_{t+1},y_{t+1}>=\sum_{ml}g_{ml}[\beta]_{m}y_{l}:=\sum_{ml}g_{ml}[\beta]_{t+1,m}y_{t+1,l}$, $g^{ml} $ is a cocycle associate with scalar product of $<[\beta]_{t+1}^{\ast},y_{t+1}^{\ast}>=\sum_{ml}g^{ml}[\beta]_{m}^{\ast}y_{l}^{\ast}:=\sum_{ml}g^{ml}[\beta]_{t+1,m}^{\ast}y_{t+1,l}^{\ast}$.
\item \emph{Type III.} Let $g_{km} $ is a cycle associate with scalar product of $<x_{t+1},[\beta]_{t+1}>=\sum_{km}g_{km}x_{k}[\beta]_{m}:=\sum_{km}g_{km}x_{t+1,k}[\beta]_{t+1,l}$, $g^{kl} $ is a cocycle associate with scalar product of their dual basis.
\item \emph{Type IV.} Let $g_{jm} $ is a cycle associate with scalar product of $<[\theta]_{t+1},[\beta]_{t+1}>=\sum_{jm}g_{jm}[\theta]_{j}[\beta]_{m}:=\sum_{jm}g_{jm}[\theta]_{t+1,j}[\beta]_{t+1,m}$, $g^{jm} $ is a cocycle associate with scalar product of their dual basis.
\end{itemize}
\end{Definition}

It is known from the differential geometry that a connection over tensor field can be written as
\begin{gather}
\bigtriangledown_{j}g_{kl}=\partial_{j}g_{kl}-\Gamma_{jk}^{m}g_{ml}-\Gamma_{jl}^{m}g_{km},
\\
\Gamma_{ij}^{m}=\frac{1}{2}g^{ml}(\partial_{j}g_{il} +\partial_{i}g_{lj}-\partial_{l}g_{ji}).
\end{gather}
The above equation allow us to work over the connection of fibre space \cite{fibre} as an extradimension in moduli stat space of Killing vector field approach. We just use the tool for the computation of expected state of our new definition of spinor field in time series data.
\begin{Definition}
Let support spinor be an arbitrage opportunity $\Gamma_{ij}^{m}$ in financial market over the moduli state space model in time series data as an Riemannian connection preserves scalar product over the parallel translation along fibre space of the physiology layer in time series data.
\end{Definition}
\begin{Definition}
The Ricci tensor in time series data is a contraction of a curvature tensor defined by $R_{ik}=R_{ikl}^{j}$ with respect to natural frame of connection
\begin{equation}
R_{ik}=\partial_{k}\Gamma_{ji}^{j}- \partial_{j}\Gamma_{ki}^{j}+\Gamma_{km}^{j}\Gamma_{ji}^{m}-
\Gamma_{jm}^{j}\Gamma_{ki}^{m}.
\end{equation}
\end{Definition}
Note that we deform space time of Kolmogorov space in time series data by using a curvature tensor and give an existence of price particle with spin and $8$-market states so called arbitron or arbitrage opportunity $R_{ik}$ in financial market. It is known from differential geometry that
\begin{equation}
\Gamma_{jk}^{j}=\frac{1}{2}g^{jm}(\partial_{k}g_{mj})=\frac{1}{2g}\frac{\partial g}{\partial x^{k}}
=\frac{\partial }{\partial x^{k}}\ln|g|^{\frac{1}{2}}
\end{equation}
since
\begin{equation}
g^{ij}=\frac{1}{g}\frac{\partial g}{\partial g_{ij}}.
\end{equation}
For the case of the market equilibrium Ricci curvature is zero $R_{ik}=0$, i.~e., the arbitrage opportunity disappears and the physiology of time series data contains no curvature. It is a Peterson-Codazzi equation in the differential form \cite{patterson,patterson2}. The solution of the equation is a curvature of SSM for classification Riemann surface of time series data.

\begin{figure}[!t]
\centering
\includegraphics[width=\linewidth]{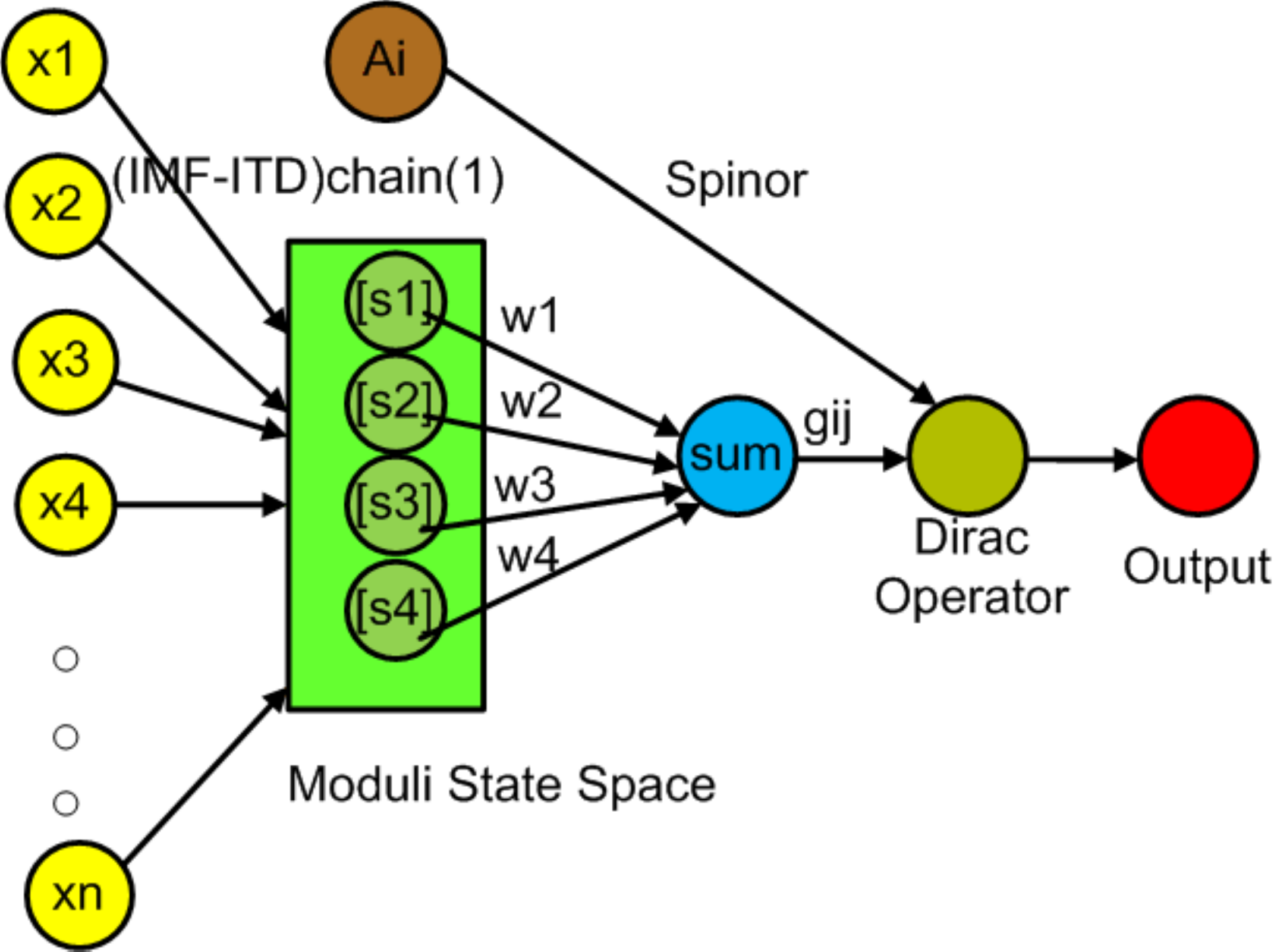}
\caption{The diagram of SSM with support Dirac machine. In comparison with the artificial neuron network we call this network the artificial support spinor network. The network is associated with Dirac equation in finances with solution with spin half particle of behavior of trader $A_{i}$ in market four states. We call the solution of a lagrangian of SSM an arbiron or arbitrage opportunity state.\label{ssm_algorithm}}
\end{figure}

We have a support Dirac machine (SDM) for financial time series data (see Fig.~\ref{ssm_algorithm} for new type of artificial neuron network with connection) by using a connection and market cocycle modulo the state of behavior of trader as bias module in neuron network,
\begin{multline}
\sum_{i} [s_{i}] \bigtriangledown_{[s_{i}]}g_{kl}<x_{t},y_{t}>=\\\gamma([A_{i}])<x_{t},y_{t}>\,\, mod\,\, [A_{i}]
\end{multline}
where $\gamma([A_{i}])$ is Dirac matrix over Pauli matrix of equivalent class of behavior of trader $[A_{i}]$, $i=1,2,3$.

The modulation of $[A_{i}]$ is equivalent to the amplitude modulation in Holo-Hilbert algorithm \cite{holo} in our empirical part of the paper (see Section~\ref{sec:results}).

\begin{Definition}\label{def}
A support spinor is a coupling state between a market four states in financial market induced from trading behavior with spin half as an optimal solution of following equation
\begin{multline}
L([s_{i}],[A_{i}],w_{i})=\sum_{i=1}^{4}[s_{i}]\bigtriangledown_{[s_{i}]}<w,[s_{i}(x_{t})]>\\-\gamma(A_{i})<w,[s_{i}(x_{t})]> \quad mod \quad [A_{i}].\label{dirac}
\end{multline}
\end{Definition}

\subsection{Proof of the existence of SSM}

\begin{Theorem}
The weight of support spinor machine is an equivalent class $[w]$ as a solution of Eq.~(\ref{dirac}), for financial time series data it can be written as
\begin{equation}
[w]= \ii\oint\limits_{H^{2}(X_{t}/Y_{t})} (\Gamma_{jk}^{m}g_{ml}+\Gamma_{jl}^{m}g_{km})d[s_{i}]\wedge d[s_{j}^{\ast}] \,\, mod \,\, [A_{i}].
\end{equation}
\end{Theorem}

\begin{proof}
Let $[w]\in [X_{t},S^{1}]$ is defined by $[w]= [\beta_{t}:\sum \lambda_{i} \frac{x_{t}}{y_{t}}   \mapsto e^{i<x_{t},y_{t}>}]= [e^{ix_{t}y_{t}}]$.
Consider type II of SSM with $g_{kl}$ defined by
\begin{multline}
g_{kl}=<w,x>=:<[\beta]_{t+1},x_{t+1}>=\\<[e^{ix_{t+1} y_{t+1}}],x_{t+1}>\simeq e^{ix_{t+1} y_{t+1}}x_{t+1}
\end{multline}
where
\begin{equation}
-i\frac{\partial}{\partial y_{t+1}}  e^{ix_{t+1} y_{t+1}}=e^{ix_{t+1} y_{t+1}}x_{t+1}:=g_{kl}
\end{equation}
We define a covector field of $\partial_{j}g_{kl}$ by a vector field
\begin{equation}
\partial_{j}g_{kl}:=\frac{\partial}{\partial y_{t+1}} g_{kl}= -i \frac{\partial^{2}}{\partial^{2} y_{t+1}}[w] \label{form}
\end{equation}
Consider Killing equation over Killing support vector field
\begin{equation}
\bigtriangledown_{j}g_{kl}=0
\end{equation}
with
\begin{equation}
\sum_{i=1}^{4}[s_{i}]\cdot\bigtriangledown_{j}g_{kl}<x_{t},y_{t}>=0.
\end{equation}
We have an equation
\begin{equation}
\bigtriangledown_{j}g_{kl}=\partial_{j}g_{kl}-\Gamma_{jk}^{m}g_{ml}-\Gamma_{jl}^{m}g_{km}=0,
\end{equation}
with
\begin{equation}
\partial_{j}g_{kl}:=-i \frac{\partial^{2}}{\partial^{2} y_{t+1}}[w]= \Gamma_{jk}^{m}g_{ml}+\Gamma_{jl}^{m}g_{km},
\end{equation}
where
\begin{equation}
\Gamma_{ij}^{m}=\frac{1}{2}g^{ml}(\partial_{j}g_{il} +\partial_{i}g_{lj}-\partial_{l}g_{ji}).
\end{equation}
If we classify the observation space $y_{t+1}\in [s_{i}]$, where $[s_{i}]$ is an equivalent class of the physiology of time series data, we will receive an equation for SSM $[w]$ by using De-Rahm cohomology in financial time series. We induce a closed surface integral over second differential form of Riemann sphere over the physiology of time series data of present state $[s_{i}]$ and expectation state $[s^{\ast}_{j}]$ (from Eq.~ (\ref{form}), integrate both sides two times and then take modulo of the state of behavior of trader)
\begin{multline}
[w]_{n=0,1,2\dots}=\\ \ii\oint\limits_{H^{2}(X_{t}/Y_{t})} (\Gamma_{jk}^{m}g_{ml}+\Gamma_{jl}^{m}g_{km})d[s_{i}]\wedge d[s_{j}^{\ast}] \quad mod\quad [A_{i}]. 
\end{multline}
Therefore we get a diophantine equation for financial time series data as a equation for SSM (we call it SSM-equation) with the quantization state of the weight $[w]_{n=0,1,2\dots}$
\begin{multline}
[w]_{n=0,1,2\dots}-n[A_{i}]=\\ \ii\oint\limits_{H^{2}(X_{t}/Y_{t})} (\Gamma_{jk}^{m}g_{ml}+\Gamma_{jl}^{m}g_{km})d[s_{i}]\wedge d[s_{j}^{\ast}] . 
\end{multline}
in moduli state space of behavior of equivalent class of a behavior of trader $[A_{i}]$. This equivalent class satisfies a differential 3-form of behavior of trader as market potential field $\mathcal{A}$
\begin{equation}
d\mathcal{A}=\sum_{ijk=1,2,3} F_{ijk}^{\bigtriangledown_{[s_{i}]}} dA_{i}\wedge A_{j}\wedge A_{k}
\end{equation}
and it has a deep relationship with supply and demand potentials in cohomology theory of financial market \cite{cohomo}.
\end{proof}

\begin{figure}[!t]
\centering
\includegraphics[width=\linewidth]{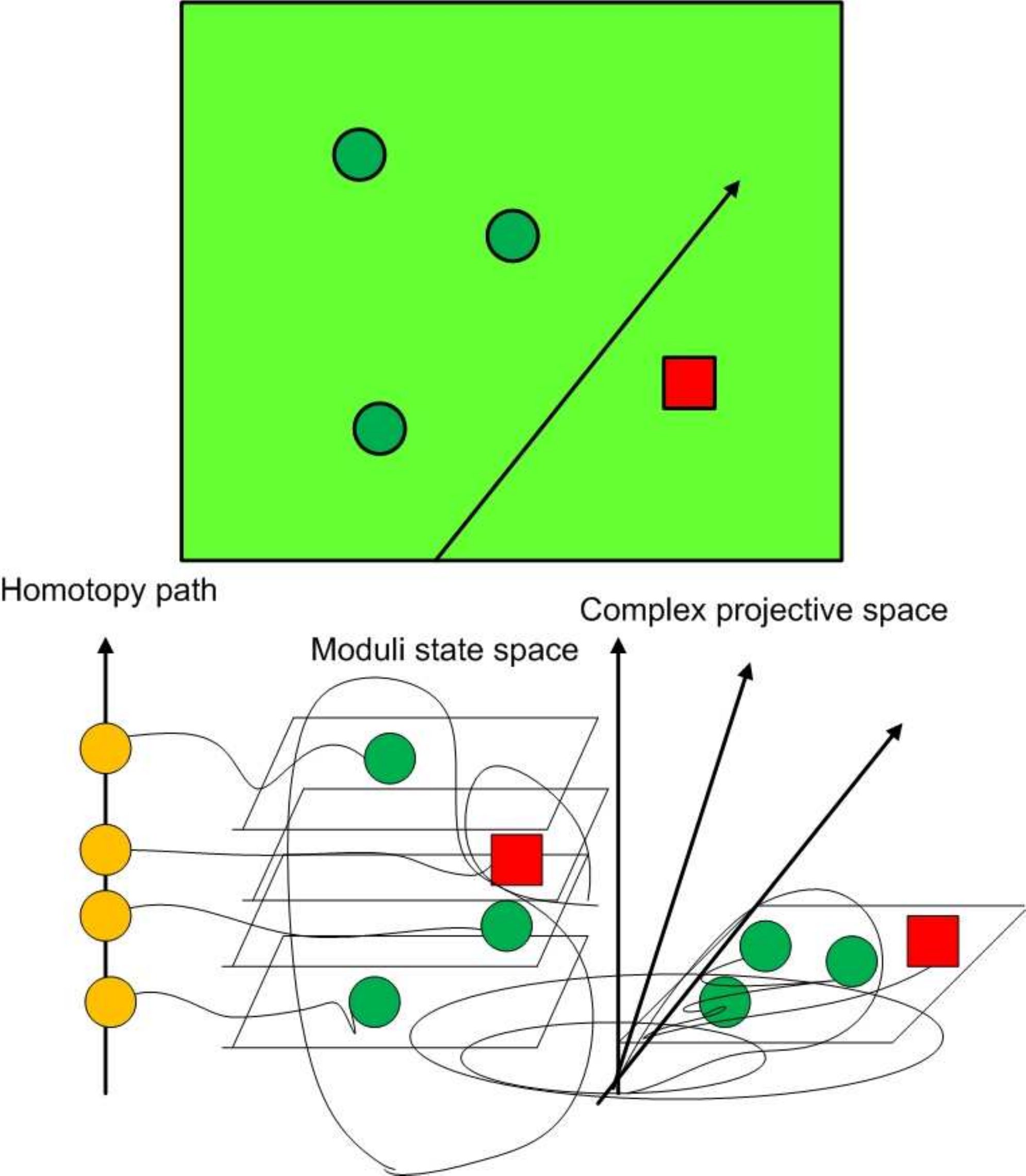}
\caption{Picture shows the extradimension in spinor field of plane separation in classical support vector machine.\label{separate}}
\end{figure}

\subsection{Optimal hyperplane for SSM}

In SSM we work on non-Euclidean space as the plane of classification in which we can use homotopy to deform classical Euclidean space of classification plane in SVM into moduli space of time series. The space of time series has intrinsic cocycle as jacobian flow in another hidden space and its dual space that embed in between plane of classification in D-brane and anti-Dbrane model for time series data.

Let us consider Hermitian product over section $\mathbb{C}P^{1}$ of hyper-K\"{a}her manifold. Let $w\in A$ be support spinor $w$ associated with a hyper-spinor plane $H_{w}=\{w\in A |w^{\ast}=w\}$ as a source of existence of weight in optimal SSM,
\begin{equation}
<w,w>=\sum_{j}<w,e_{j}^{\ast}><e_{j},w>.
\end{equation}
We work on hidden space of time series data with induced dual Lagragian $L^{\ast}$ with free product over von Neumann algebras,
\begin{equation}
L^{\ast}(w^{\ast},b^{\ast},\alpha^{\ast})\otimes_{A} L(w,b,\alpha).
\end{equation}
We set extra parameter called a Lagrange multiplier $\alpha_{i}=\frac{dg_{ij}}{dt}$ as a Jacobian flow of a extradimension of hidden space of time series $X_{t}$ and we introduce $\alpha_{i}^{\ast}=\frac{dg^{ij}}{dt}$as a Jacobian flow of a extradimension of hidden space of time series $X_{t}^{\ast}.$ The space of input data $x_{t}\in X_{t}$ is endown with $T_{0}-$separation over quotient topology of moduli space of induce a Teichmüller space in time series data with genus $g=g_{ij}(x_{t})$. We use D-brane theory of double D-brane.One side of brane is a modelling time series data as classifying space with separation with non-Eculidean plane with negative curvature as geometrical meaning of SSM. If input data of feature space $(x_{i},y_{i}=1)$, we associate with what spinor machine do by using homotopy to deform underlying topological space (Riemann surface of time series data) to target Riemann sphere with positive curvature. We deform by Jacobain flow or derivative of Riemann metric tensor $g_{ij}$ along parallel transport of Killing vector field.

If a curvature of the plane of classification of SSM is not zero, we deform moduli space to hyperbolic space  
with induced classification function for SSM classify with negative curvature $g_{ij}=-1$
\begin{equation}
g_{ij}<w,x>=-<w,x>=||w||.||x||\cosh \dd(w,x)
\end{equation}

We normalized support spinor and input state vector to unit length in $S^{1}$ according to our definition above, so we have a distance from complex plane, $[w]\in H_{0}(X_{t}/Y_{t})$ project to projective line by $<w,x>\in \mathbb{R}$ with $<w,x>:=[\tan\theta_{t}] =[\frac{w}{x}]\simeq 
\frac{[\m{Im}(y_{t})]}{[\m{Re}(y_{t})]}\in\mathbb{C}P^{1}\simeq S^{2}$. Where $[\m{Im}(y_{t})]$, $[\m{Re}(y_{t})]$ are equivalent class of real as market cocycle and imaginary part as market boundary  in cohomology theory in financial market. This scalar product has a Hermitian structure of complex projective space.  We define a market cocycle over Hermitian product of section of Kähler manifold of financial market by
\begin{equation}
g_{ij}':=d(w,x)=\arccosh(-<w,x>).
\end{equation}
 
We consider $\partial_{k}g_{ij}=0$ as a simple solution of an example of how to find a separated hyperplane of SSM over complex projective plane (see Fig.~\ref{separate} for the projective hyperplane) with previous definition of $\partial_{k}:=\frac{\partial}{\partial\theta_{t+1}}$. We will not compute all components of Riemannian curvature here, because the equation will be very cumbersome. We just want to shown the simple solution without the diffinition of Riemannian curvature of parallel translation of support vector along curves on a Kähler manifold.

\begin{figure}[!t]
	\centering
	\includegraphics[width=\linewidth]{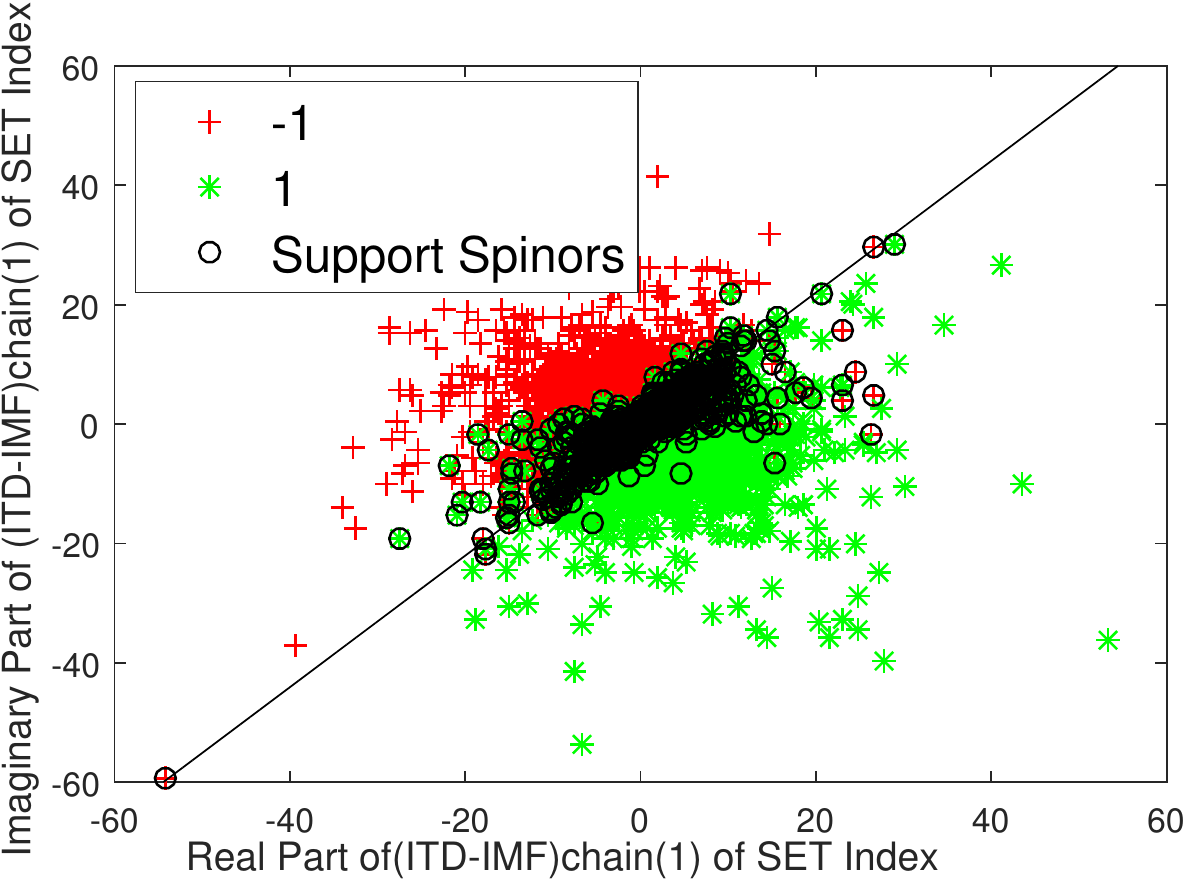}
	\caption{The empirical analysis of SSM over complex plane $\ICHAIN$ of closed price for financial time series data of SET index in 9528 days from 02/05/1975 to 24/06/2014. The separated hyperplane between up and down state of $y_{t}=\{-1,1\}$ with linear kernel of HT-SSM is approximately a straight line cut unit cycle $S^{1}$ at $\frac{\pi}{4}$ and $\frac{3\pi}{4}$. This line separate the up state as the equivalent class of $[1]:=[2\pi]\in Y_{1}:=\pi_{1}(\mathbb{R}^{2}-X_{t})$ and the down state defined by the equivalent class of $[-1]:=[\pi]$.
		\label{hyperplane}}
\end{figure}
 
Consider
\begin{multline}
\partial_{k}g_{ij}':=\frac{\partial}{\partial\theta_{t+1}}\arccosh(-<w,x>) =\\ \frac{1}{\sqrt{<w,x>^{2}-1}}\frac{\partial}{\partial\theta_{t+1}}(-<w,x>)=0.
\end{multline}
We have $<y_{t},x_{t}>\in [\frac{y_{t}}{x_{t}},1]\simeq <w,x>\in \mathbb{C}P^{1}\simeq S^{2}$ so we approximate the scalar curvature of weight $<w,x>$ approximated by  $[<y,x>]\in \mathbb{CP}^{1}$ since weight of SSM use for classifying $y_{t}\in Y_{t}$. We assume that the solution is perfect approximation. In the above scalar product of weight of SVM with parallel translation along curve in $S^{2}$, we induce a Riemannian curvature in SSM over principle bundle $\mathbb{C}P^{1}$ of Kolmogorov space in time series data by
\begin{multline}
\frac{1}{\sqrt{<w,x>^{2}-1}}\frac{\partial}{\partial\theta_{t+1}}(-<w,x>)=\\\frac{1}{\sqrt{<w,x>^{2}-1}}\frac{\partial}{\partial\theta_{t+1}}(-\tan\theta)=0.
\end{multline}
 We get
\begin{equation}
(<w,x>-1)(<w,x>+1)=0
\end{equation}
and $\tan\theta=0$.
We consider
\begin{equation}
<w,x>\simeq  \frac{w}{x}\simeq \tan\theta \simeq \frac{[\m{Im}(y_{t})]}{[\m{Re}(y_{t})]}\in \mathbb{CP}^{1}.
\end{equation}
Since $(<w,x>-1)(<w,x>+1)=0$ we get an equivalent class of classification over degree of map in complex plane by given a simple example of solution a hyperplane with equivalent class of weight in spinor of SSM with solution in principle angle,
\begin{equation}
[\theta^{\ast}]=\arctan\frac{[\m{Im}(y_{t})]}{[\m{Re}(y_{t})]}=0, \frac{\pi}{4}, \frac{3\pi}{4},2\pi. 
\end{equation}
We choose an optimal separate hyperplane for SSM as a support spinor with degree $[\theta^{\ast}]=[\frac{\pi}{4}]$. The of result of empirical analysis for SSM is shown in Fig.~\ref{hyperplane}.

\begin{figure}[!t]
	\centering
	\includegraphics[width=.9\linewidth]{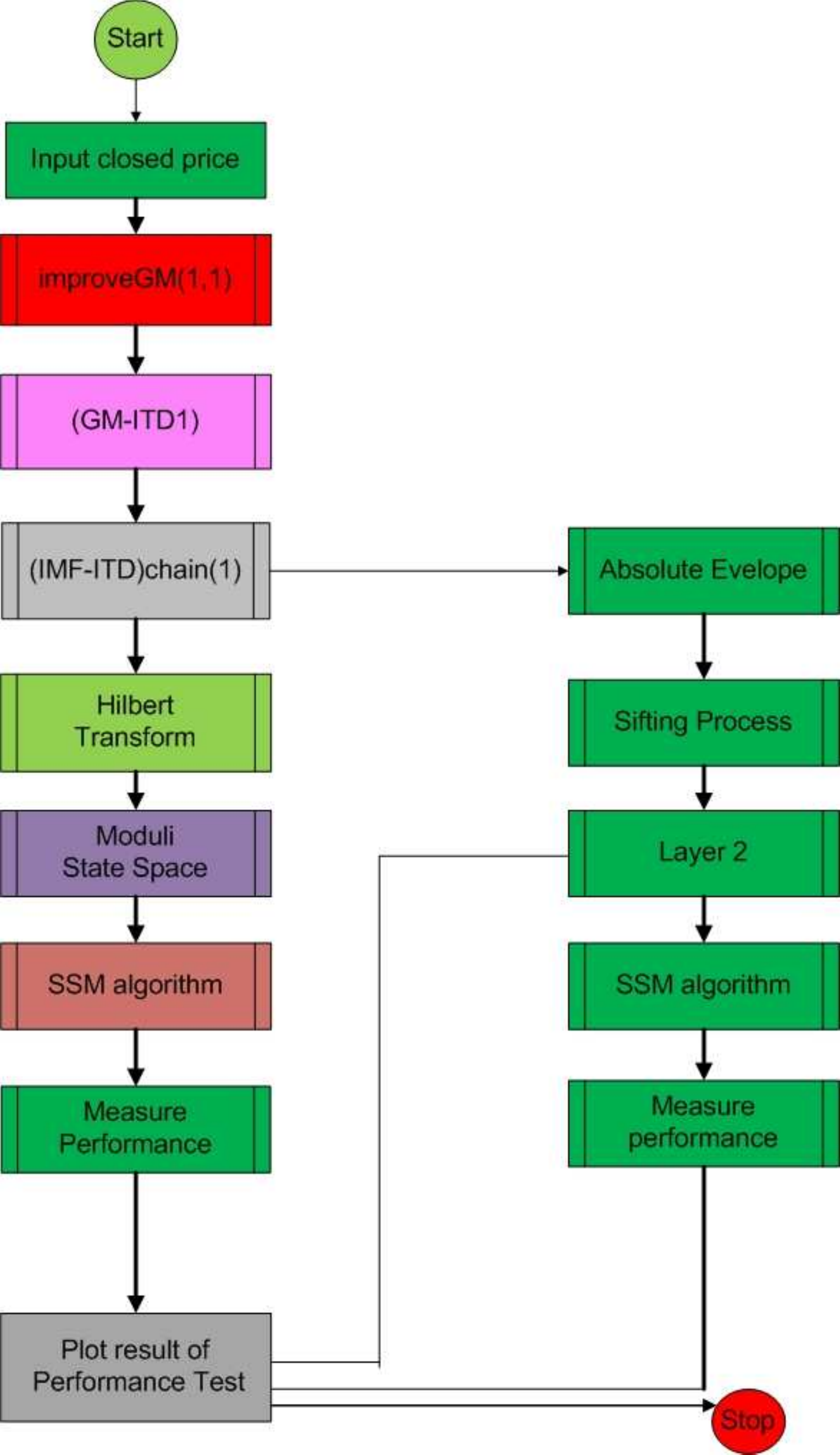}
	\caption{The flowchart of all involved modules for the performance measuring between SSM and SVM in this work \label{flow}.}
\end{figure}

\section{Result of Empirical Analysis}\label{sec:results}

In this section we implement the SSM with the adaptive data analysis algorithm. The end points of time series data are classified after empirical mode decomposition into four states in observation state space model $Y_{t}=Y_{t}([s_{1}],[s_{2}],[s_{3}],[s_{4}])$. The hidden state space of behavior of trader $X_{t}$ is in the complex projective space over second layer of the Holo-Hilbert spectral analysis with the amplitude modulation mode. We test the performance of the SSM to classify the state in financial time series data only for the up and down states with 6 different types of financial time series data (for details about the empirical test data see Table~\ref{input3}). The performance of directional prediction for various kernels is compared with the classical SVM methodology.

The sifting process of the empirical mode decomposition is one of an example of moduli state space model since mean of sifting process can be approximate by identity of an equivalent class over physiology of time series data $[s_{2}]+[s_{4}]=[1]\simeq [s_{3}]+[s_{1}]$. It is a moduli of mean in sifting process of intrinsic mode functions (IMF). The Holo-Hilbert transform of the result of $\ICHAIN$ (defined in \cite{cw}) to a complex plane is equivalent to the equivalent class of cohomology group in time series data in this work. The Holo-Hilbert algorithm renders a spinor field in higher extradimension. It is assumes that the result of 2nd-layer is an angle of support spinor field as a solution of the Eq. (\ref{dirac}) in Definition \ref{def}.
						 
\begin{table*}[thbp]
\renewcommand{\arraystretch}{1.3}
\centering
\begin{tabular}{llclrcr}
\toprule
\textbf{Stock index} & \textbf{Country} &\textbf{Data Set}& \textbf{Sample type} & \textbf{Size} & \textbf{Date} & \textbf{Date number}\\
\colrule
SET&Thailand &(1)& Out of sample test &651& 26/06/2014 --- 14/10/2016 &9530--10181\\
&&(2)&Training set&9529& 02/05/1975 --- 24/06/2014 &1--9528 \\ 
&&(3)& Out of sample& 31 & 26/06/2014 --- 13/08/2014 &9530--9560\\
RTS& Russia &(1)& Out of sample test &31& 03/12/2015 --- 20/01/2016 &5019--5049 \\
&&(2)&Training set&5018& 01/09/1995 --- 02/12/2015 &1--5018   \\ 
&&(3)& Out of sample&142& 03/12/2015 --- 01/07/2016  &     5019--5160  \\
FBM KLCI& Malaysia&(1)& Out of sample test &99& 03/03/2016 --- 25/07/2016 &8416--8514\\
&&(2)& Out of sample test &31& 03/03/2016 --- 14/04/2016 &8416--8446\\
&&(3)&Training set&8415& 04/03/1982 --- 02/03/2016  &1--8445 \\ 
PSEi& Philipines&(1)& Out of sample test &31& 03/03/2016 --- 18/04/2016 &7192--7222\\
&&(2)&Training set&7191& 02/03/1987 --- 02/03/2016 &1--7191 \\ 
&&(3)& Out of sample&99& 03/03/2016 --- 25/07/2016    &7192--7290\\
JKSE& Indonesia&(1)& Out of sample test &31&  25/02/2016 --- 12/04/2016 & 6327--6358\\
&&(2)&Training set&6356& 06/04/1990 --- 24/02/2016 & 1--6356 \\ 
&&(3)& Out of sample&96&  25/02/2016 --- 20/07/2016 & 6327--6422\\
STI& Singapore&(1)& Out of sample test &31& 04/03/2016 --- 20/04/2016 & 4145--4176\\
&&(2)&Training set&4144 & 31/08/1999 --- 03/03/2016  &1--4144 \\ 
&&(3)& Out of sample&85 & 04/03/2016 --- 04/07/2016  &4145--4229\\
\botrule
\end{tabular}
\caption{Table shows six different types of testing time series data for stock indexes of closed prices.\label{input3}}
\end{table*}

\begin{table*}[!t]
\renewcommand{\arraystretch}{1.3}
\centering
\begin{tabular}{llcccccrrc}
\toprule
\textbf{Stock} & \textbf{Country} & \textbf{Days} & \textbf{Date} &\textbf{Kernel}&  \multicolumn{2}{c}{\textbf{Performance}}  & \multicolumn{2}{c}{\textbf{Profit}} & \textbf{Improved}\\ 
\textbf{index} & &  &  &\textbf{type}& \textbf{SVM} & \textbf{SSM} & \textbf{SVM} & \textbf{SSM} & \textbf{performance}\\ 
\colrule
SET&Thailand&31& 9530-9560& rbs(0.05)& 43.33\%&67.74\%& -1.73 & 116.48&  56.34\%\\
&& 31&9530-9560& linear& 33.33\%& 58.06\%&-90.25& 50.82& 74.20\%
\\
RTS&Russia& 31&5019-5049  &rbs(0.05)& 58.06\%&67.74\%&-27.65 &113.55&16.67\%\\
&& 31&5019-5049&linear&41.94\%&61.29\%&-62.41&64.71&   46.13\%\\
FBM KLCI&Malaysia&31 &8416-8446&rbs(0.05)& 56.25\%&58.06\%& 19.80&22.25&3.22\%\\
&& 31&8416-8446&linear&46.88\%&38.71\%&-36.96&-31.29&-17.49\%\\
 PSEi&Philiphines& 31 &7192-7222&rbs(0.05)&48.39\%&61.29\%& -189.03&376.49& 26.66\%\\
&&31  &7192-7222&linear&41.94\%&61.29\%& -198.97&281.39&46.14\%\\
JKSE&Indonesia& 31 &6327-6358& rbs(0.05)&54.84\%&58.06\%& 100.72&129.25&5.87\%\\
&& 31 &6327-6358&linear&41.94\%&61.30\%& -129.25&187.36&46.19\%\\
STI&Singapore&31&4145-4176&rbs(0.05)& 48.39 \%&48.39\%  &-25.31&62.89&0\%\\ 
&&31&4145-4176&linear  &  45.16\% & 46.67\%&-33.99  &-183.44&0.007\%\\
\colrule
Average&&&&&46.70\%&57.38\%&-56.25&  98.71& 25.33\%\\
\botrule
\end{tabular}
\caption{Table shows six different types of test time series data for stock indexes with SVM forecasting. The table reports the performance of SVM with different kernels in Euclidean plane.
\label{input}}
\end{table*}

\begin{table*}[!t]
	\renewcommand{\arraystretch}{1.3}
	\centering
	\begin{tabular}{llcccccclc}
		\toprule
		\textbf{Stock} & \textbf{Country} & \textbf{Days} & \textbf{Date} &\textbf{Kernel}&  \multicolumn{2}{c}{\textbf{Performance}}  & \multicolumn{2}{c}{\textbf{Profit}} & \textbf{Improved}\\ 
		\textbf{index} & &  &  &\textbf{type}& \textbf{HT-SVM} & \textbf{HT-SSM} & \textbf{HT-SVM} & \textbf{HT-SSM} & \textbf{performance}\\ 
		\colrule
		SET&Thailand&31& 9530-9560& rbs(0.05)&58.06\%&67.74\%& 30.20 &  116.48& 16.67\%\\
		&& 31&9530-9560& linear&51.61\%& 58.06\%&-2.84 &50.82 &12.50\%\\
		RTS&Russia& 31&5019-5049  &rbs(0.05)&61.29\%&67.74\%&55.89 &113.55& 10.52\%\\
		&& 31&5019-5049&linear&51.61\%&61.29\%&46.67&64.71&18.76\%\\
		FBM KLCI&Malaysia&31 &8416-8446&rbs(0.05)&54.84\%&58.06\%& -5.09&22.25&5.87\%\\
		&& 31&8416-8446&linear&41.94\%&38.71\%&-26.41&-31.29& -7.70\%\\
		PSEi&Philiphines& 31 &7192-7222&rbs(0.05)&58.06\%&61.29\%&215.45&376.49&5.58\%\\
		&&31  &7192-7222&linear&45.16\%&61.29\%&-240.53&281.39&35.72\%\\
		JKSE&Indonesia& 31 &6327-6358& rbs(0.05)&70.97\%&58.06\%& 495.90&129.25&-18.19\%\\
		&& 31 &6327-6358&linear&58.06\%&61.30\%& 169.22&187.36&5.58\%\\
		STI&Singapore&31&4145-4176&rbs(0.05)& 45.16 \%&48.39\%  &16.13&62.89& 7.15\%\\ 
		&&31&4145-4176&linear  & 44.44\% & 46.67\%&-60.20  &-183.44&5.02\%\\
		\colrule
		Average&&&&&53.43\%&57.38\%&57.45&98.71&8.12\%\\
		\botrule
	\end{tabular}
	\caption{Table shows six different types of test time series data for stock indexes with one day ahead prediction of HT$-\ICHAIN-$SSM methodology compared with HT$-\ICHAIN-$SVM.  The separation is done over the complex plane by using Hilbert transform. SSM has a better performance in directional one day ahead prediction than SVM.\label{input2}}
\end{table*}

\begin{table*}[!t]
\renewcommand{\arraystretch}{1.3}
\centering
\begin{tabular}{llcccccrrc}
\toprule
\textbf{Stock} & \textbf{Country} & \textbf{Days} & \textbf{Date} &\textbf{Kernel}&  \multicolumn{2}{c}{\textbf{Performance}}  & \multicolumn{2}{c}{\textbf{Profit}} & \textbf{Improved}\\ 
\textbf{index} & &  &  &\textbf{type}& \textbf{HT-SVM} & \textbf{HT-SSM} & \textbf{HT-SVM} & \textbf{HT-SSM} & \textbf{performance}\\ 
\colrule
SET&Thailand&651& 9530-10181& rbs(0.01)&52.15\%&54.62\%&531.34 &584.94&4.74\%\\
&& 651&9530-10181& linear&48.77\%&51.23\%&30.5&298.06&5.04\%\\
RTS&Russia& 142&5019-5160  &rbs(0.05)& 49.30\%&54.93\%&  117.62 & 299.88& 11.42\%\\
&& 142&5019-5160&linear&54.22\%&50.70\%& 432.20& -17.06&-6.49\%\\
FBM KLCI&Malaysia &99&8416-8514&linear&43.43\%&47.47\%& -106.60&13.32& 9.30\%\\
&& 99&8416-8514&rbs(0.01)&41.42\%&  53.54\%& -115.90&100.34&29.26\%\\
PSEi&Philiphines& 99 &7192-7290&rbs(0.05)&53.53\%& 58.59\%&  -17.19& 1342.53&9.45\%\\
&&99  &7192-7290&linear&57.57\%&57.57\%&  1381.27& 1204.59&0\%\\
JKSE&Indonesia& 96 &6327-6422& rbs(0.05)& 59.37\%&62.50\%&  735.46&731.66& 3.13\%\\
&& 96 &6327-6422&linear&51.04\%& 60.42\%& 58.67& 390.05&5.58\%\\
STI&Singapore&85&4145-4229&rbs(0.05)&   52.94\%& 49.41\%  & 304.84&  217.5&  -6.67\%\\ 
&&85&4145-4229&linear  &  44.71\% & 50.59\%& -205.16 &  -82.94& 13.15\%\\
\colrule
Average&&&&&50.70\%&54.29\%& 262.20&423.57&6.49\%\\
\botrule
\end{tabular}
\caption{Table shows six different types of test time series data for stock indexes with one day ahead prediction of HT$-\ICHAIN-$SSM methodology compared with HT$-\ICHAIN-$SVM. The separation is done over the complex plane by using Hilbert transform. Again SSM has better performance in one day ahead directional prediction than SVM. \label{input4}}
\end{table*}

In deep, the Holo-Hilbert transform is related with a new concept of cohomology sequence in extradimension of Kolmogorov space in time series data. Let $x(t)\in \mathbb{R}$ be
\begin{equation}
x(t)=\m{Re} \sum_{j=1}^{N}a_{j}\ee^{\ii 2\pi f_{j}t},  
\end{equation}

The expansion above is based on an adaptive IMF basis, so we have \cite{holo}
\begin{equation}
x(t)=  \sum_{j=1}^{N}c_{j}(t) = \m{Re} \sum_{j=1}^{N}a_{j}\ee^{\ii   \int_{t}  \omega_{j}(\tau)\dd\tau},  
\end{equation}

The algorithm (Fig.~\ref{flow}) to find a Holo-Hilbert spectrum is iterative over a few steps:
\begin{enumerate}
\item Take the absolute value of the $\ICHAIN$.
\item Identify all maxima of the absolute-valued function of $\ICHAIN$.
\item Construct the envelope by a natural spline through all the maxima.
\item Do $\ICHAIN$ transform over the envelope of layer from the previous step. 
\item Repeat step 1. again for the next layer of amplitude modulation process.
\end{enumerate}
The procedure above is applied separately to the amplitude modulation (AM) and the frequency modulation. The result of Holo-Hilbert amplitude hidden layer is obtained as an extradimension presentation which we can use to detect market crash and classify next state of prediction. The nested expression for the amplitude function has the form
\begin{equation}
a_{j}(t)=\sum_{k}[\m{Re}\sum_{l}a_{jkl}(t)\ee^{\ii \int_{t} \omega_{l}(\tau)\dd\tau} \cdots  ]\ee^{\ii   \int_{t}  \omega_{k}(\tau)d\tau}.
\end{equation} 
The Holo-Hilbert spectrum is a high-dimensional spectrum. The presentation is on the four layers AM spectra modulation for price prediction. 
 
\begin{table*}[!thb]
\renewcommand{\arraystretch}{1.3}
\centering
\begin{tabular}{llllrlr}
\toprule
\textbf{Rank} & \textbf{SET Symbol} & \textbf{Securities Name} & \textbf{SVM} & \textbf{Profit} & \textbf{SSM} & \textbf{Profit} \\
\colrule
1 &ADVANC &ADVANCED INFO SERVICE   &43.47\%&-5&71.73\%&225\\
2 &AOT &AIRPORTS OF THAILAND   &41.43\% &-10.9&51.44\%&30  \\
3 &BANPU &BANPU   &36.23\%&1.45 & 47.10\%&12.9 \\
4 &BAY &BANK OF AYUDHYA   &48.55\%&0.00&38.41\% & 8.0\\
5 &BBL &BANGKOK BANK   &37.68\% &-26.5&36.23\%&7.0\\
6 &BCP &THE BANGCHAK PETROLEUM   &44.20\% &3.5& 51.45\%&  27.75\\  
7 &BEC &BEC WORLD   &41.30\% &-2.5&42.02\%&9.0 \\
8 &BDMS &BANGKOK DUSIT MEDICAL SERVICES   &44.93\% & -0.6&60.14\%& 12.4\\
9 &BH &BUMRUNGRAD HOSPITAL   &44.43\% &-0.6& 64.49\% &255.0\\
10 &BIGC &BIG C SUPERCENTER   & 36.93\%&-29& 50.72\%&72.0 \\
11 &BTS &BTS GROUP HOLDINGS   &42.75\% &0.55&49.27\%&3.6\\  
12 &CPALL &CP ALL   &36.23\% &-15&48.55\%&20\\
13 &CPF &CHAROEN POKPHAND FOODS   &34.78\% &-1.6&45.65\%&6.6\\
14 &CPN &CENTRAL PATTANA   &  49.27\%  & 11.75& 49.27\%&13.75\\
15 &DELTA &DELTA ELECTRONICS (THAILAND)   &44.93\%&13.75&60.87\%&114.75\\
16 &EGCO &ELECTRICITY GENERATING   &43.48\% &6  &57.25\%&81\\
17 &GLOW &GLOW ENERGY   &46.37\% &6 &63.77\%&115.25 \\
18 &HMPRO &HOME PRODUCT CENTER   &40.58\% &-0.39&54.35\%&5.09\\
19 &IRPC &IRPC   &39.85\% &-1.18  &42.03\%&0.04\\
20 &KBANK &KASIKORNBANK   &36.95\%&-2&38.40\%&-70\\
21 &KKP &KIATNAKIN BANK   &39.13\% &0.75&42.03\%&35\\
22 &KTB &KRUNG THAI BANK   &47.10\% &7&51.45\%&8\\
23 &LH &LAND AND HOUSES   &42.05\% &0.8&45.65\%&3.45\\
24 &MAKRO &SIAM MAKRO   &31.88\% &1&36.23\%&13\\
25 &MINT &MINOR INTERNATIONAL   &51.43\%&13&56.24\%&24.93\\
26 &PS &PRUKSA REAL ESTATE   &40.58\% &0.25&63.77\%&35.95\\
27 &PTT &PTT   &42.75\% &52  &49.28\%&172.00\\
28 &PTTEP &PTT EXPLORATION AND PRODUCTION   &40.58\% &-34&35.51\%&-78\\  
29 &RATCH &RATCHABURI ELECTRICITY GENERATING HOLDING  &39.13\% &-1.5&50.00\%&2.6\\  
30 &ROBINS &ROBINSON DEPARTMENT STORE   &43.48\% &-2.25&48.55\%&11\\
31 &SCB &THE SIAM COMMERCIAL BANK   &49.27\% &47&28.26\%&-97\\
32 &SCC &THE SIAM CEMENT   &50.00\%   &16&55.07\%&18.40\\
33 &SCCC &SIAM CITY CEMENT   &43.48\%  &105&51.45\%&103\\
34 &SSI &SAHAVIRIYA STEEL INDUSTRIES   &18.00\% &-0.6&32.60\%&0.24\\
35 &STA &SRI TRANG AGRO-INDUSTRY   & 39.10\% & 3&47.10\%&8.6\\
36 &TCAP &THANACHART CAPITAL   &36.97\% &2&51.44\%&11.5\\
37 &THAI &THAI AIRWAYS INTERNATIONAL   &38.40\% &1.9&36.23\%&-1.8\\  
38 &TMB &TMB BANK   &37.68\% &0.3&52.89\%&2.22\\
39 &TOP &THAI OIL   &40.48\% &-8.25 &44.92\%&6.25 \\
40 &TPIPL &TPI POLENE   &47.83\%  &0.37&53.62\%&2.88\\ 
41 &TRUE &TRUE CORPORATION   &44.20\% &-0.6 &50.00\%&7.4 \\
42 &TU &THAI UNION GROUP   &39.13\%  & -0.53&63.00\%&25.8\\
\colrule
\multicolumn{3}{l}{Average performance and sum of profit} &38.55\% &158.91&44.21\%&1234\\
\botrule
\end{tabular}
\caption{Table shows 42 stocks in SET50 Index Futures of Thai stock market.	The training data with $\ICHAIN$ transformation have run in out of sample for 2000 trading days over close price between the dates 23/08/2007 to 29/10/2015. An out of sample test set for SVM with linear kernel has run over 138 days between 19/08/2014 to 10/03/2015. An out of sample test set for SSM with linear kernel has run over 138 days between 09/04/2015 to 02/11/2015. The better performance of SSM vs. SVM is evident from the empirical result of one day ahead directional forecasting.}
\label{42stock}
\end{table*}

\begin{figure*}[!t]
	\centering
	\includegraphics[width=.49\linewidth]{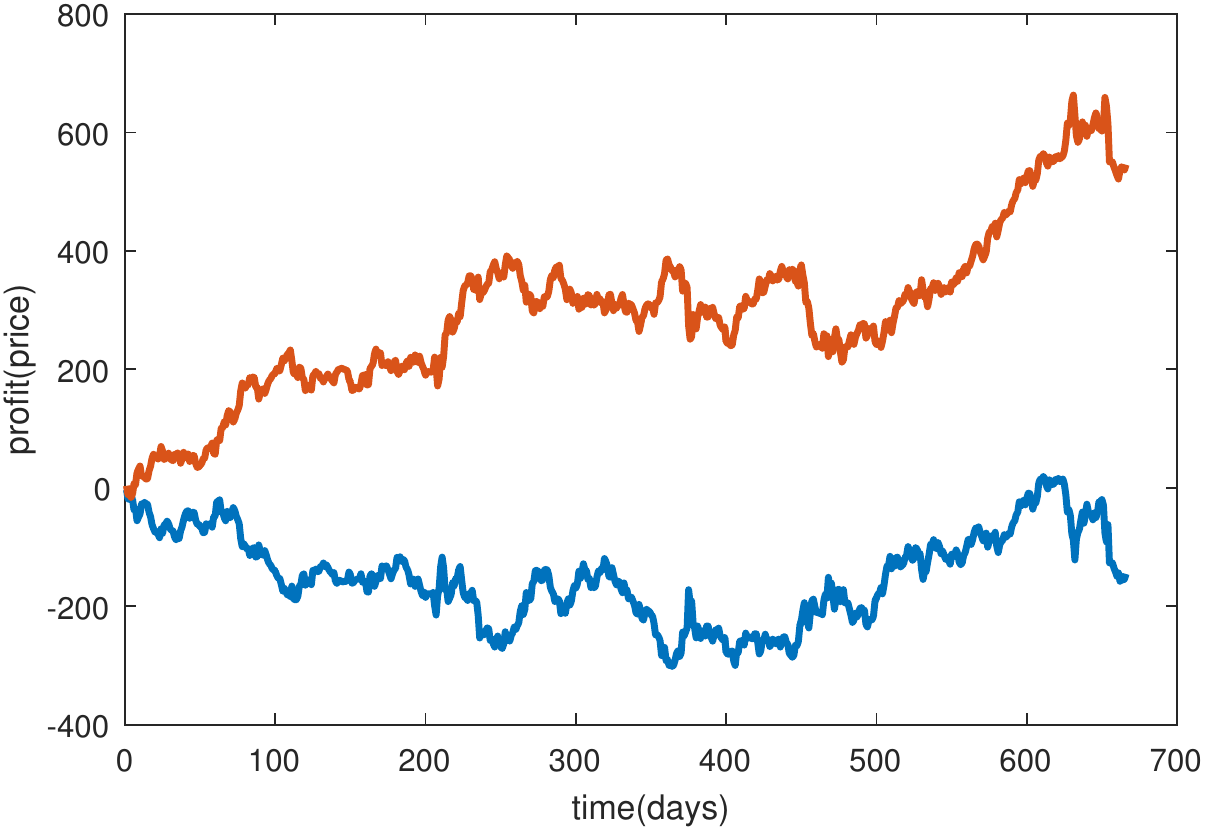}
	\includegraphics[width=.49\linewidth]{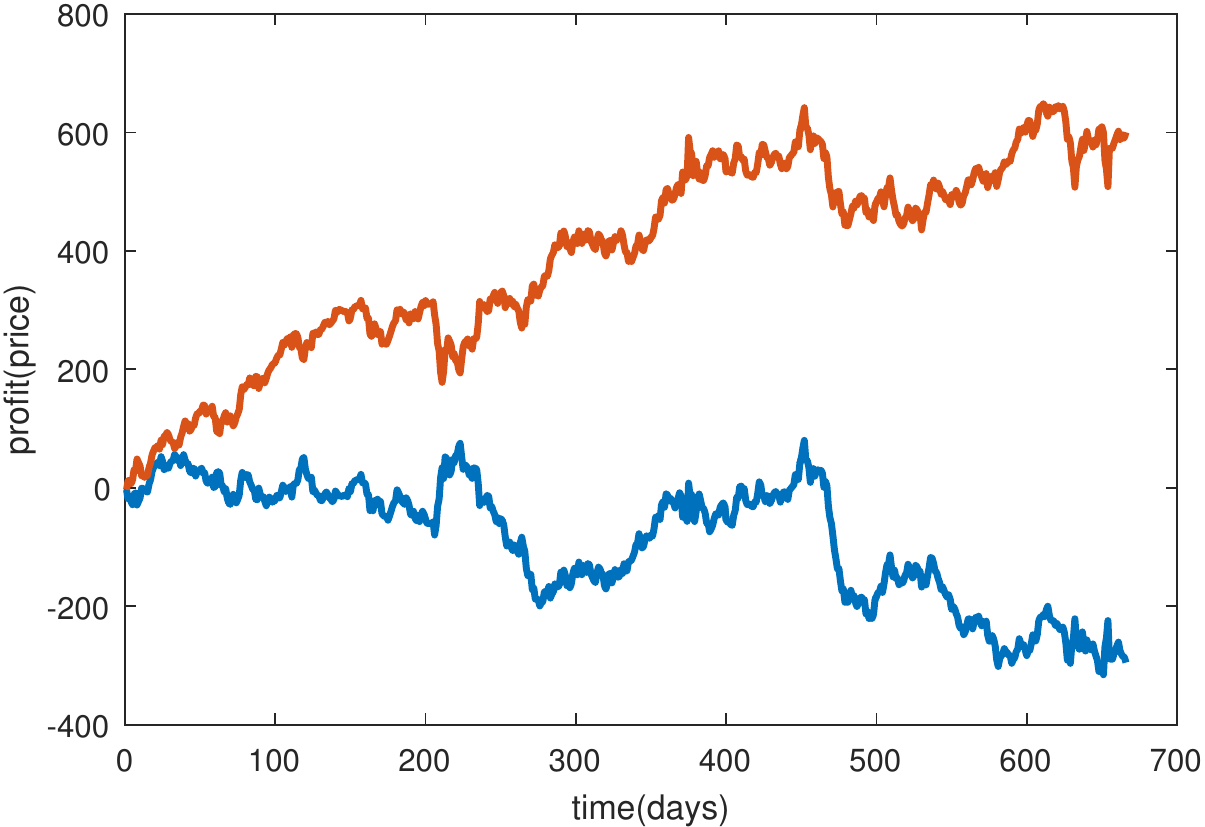}
	\caption{On the left: The comparison between  the performance of SVM (blue) and SVM with Hilbert transform of $\ICHAIN$ (red) with the linear kernel of SVM. The performance of  HT-SVM is $51.28\%$ with profit $485.86$, the performance of traditional SVM is $49.77\%$ with lost $-145.34$.
	On the right: The comparison between the performance of SSM (blue) and SSM with Hilbert transform of $\ICHAIN$ (red). We used rbs kernel(0.001) for HT-SSM, with the performance of directional prediction $55.19\%$ and the profit $606.84$. The blue line describes classical SVM with rbs kernel(0.05) with the performance of directional prediction just $47.67\%$ with the lost $-294$.
	For both cases the SET data are run out of sample in one day ahead directional prediction for 665 trading days starting from 26/6/2014 to 2/11/2016. 
	\label{result_thailand}}
\end{figure*}

\begin{figure*}[!t]
	\centering
	\includegraphics[width=.49\linewidth]{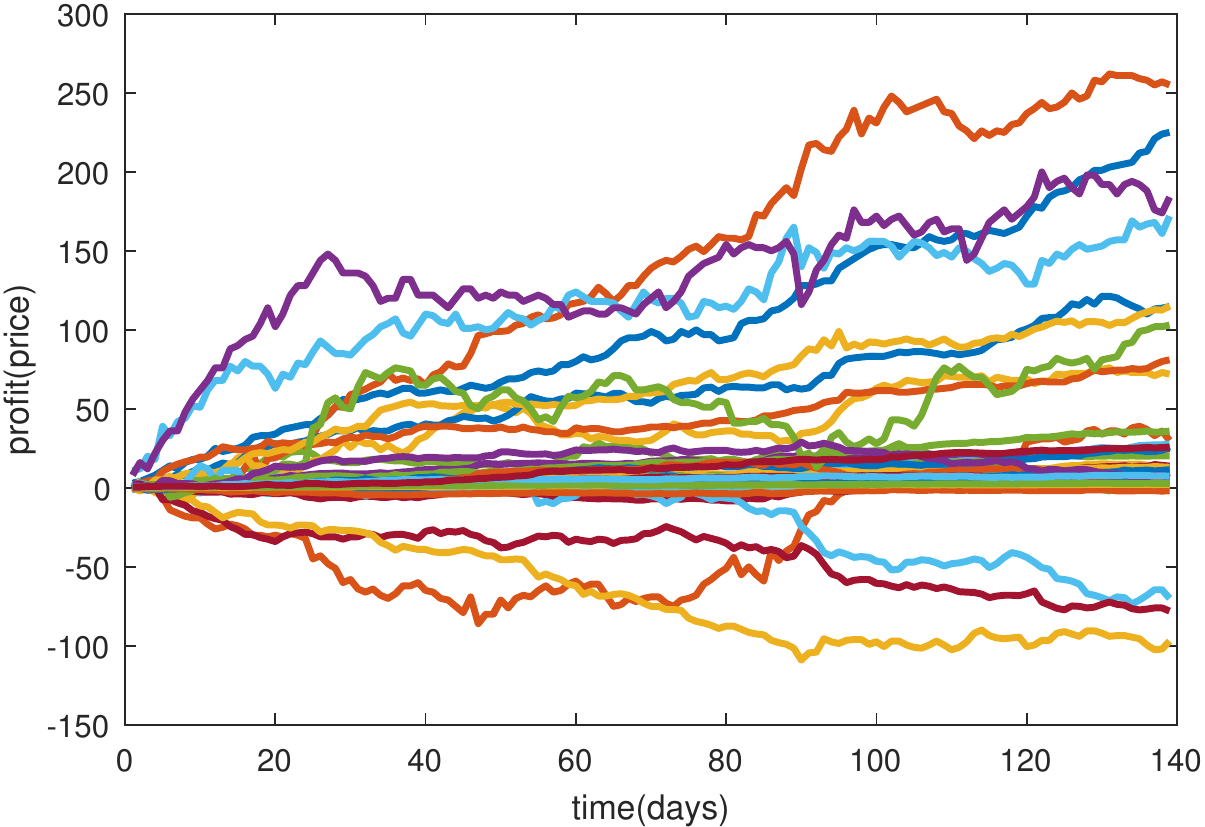}
	\includegraphics[width=.49\linewidth]{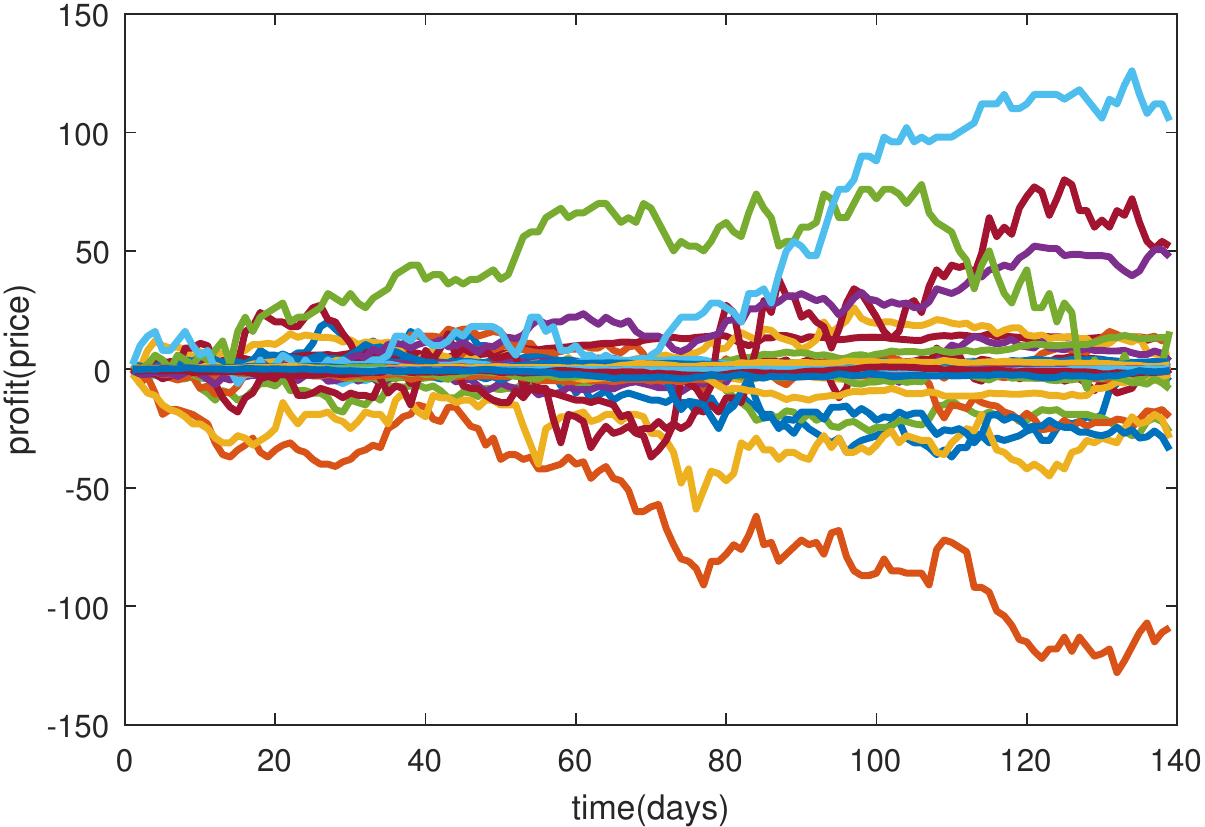}
	\caption{The profit of 42 stocks in SET50 Index Futures for SSM and SVM according Table~\ref{42stock}. The data run out of sample in one day ahead directional prediction for 138 days and trained with 2000 days starting from 23/08/2007 to 29/10/2015. On the left: The profit of one day ahead directional prediction  for SSM with a linear kernel and Hilbert transform of $\ICHAIN$.  The color lines are the earing prices after trading with the result of directional prediction. The average performance of SSM is just $44\%$ because of too short number of training days, but the sum of profit of within 138 days is 1234 points.
	On the right: The profit of one day ahead directional prediction for SVM with a linear kernel without $\ICHAIN$. \label{result_42stock}} 
\end{figure*}

\begin{figure}[!th]
	\centering
	\includegraphics[width=\linewidth]{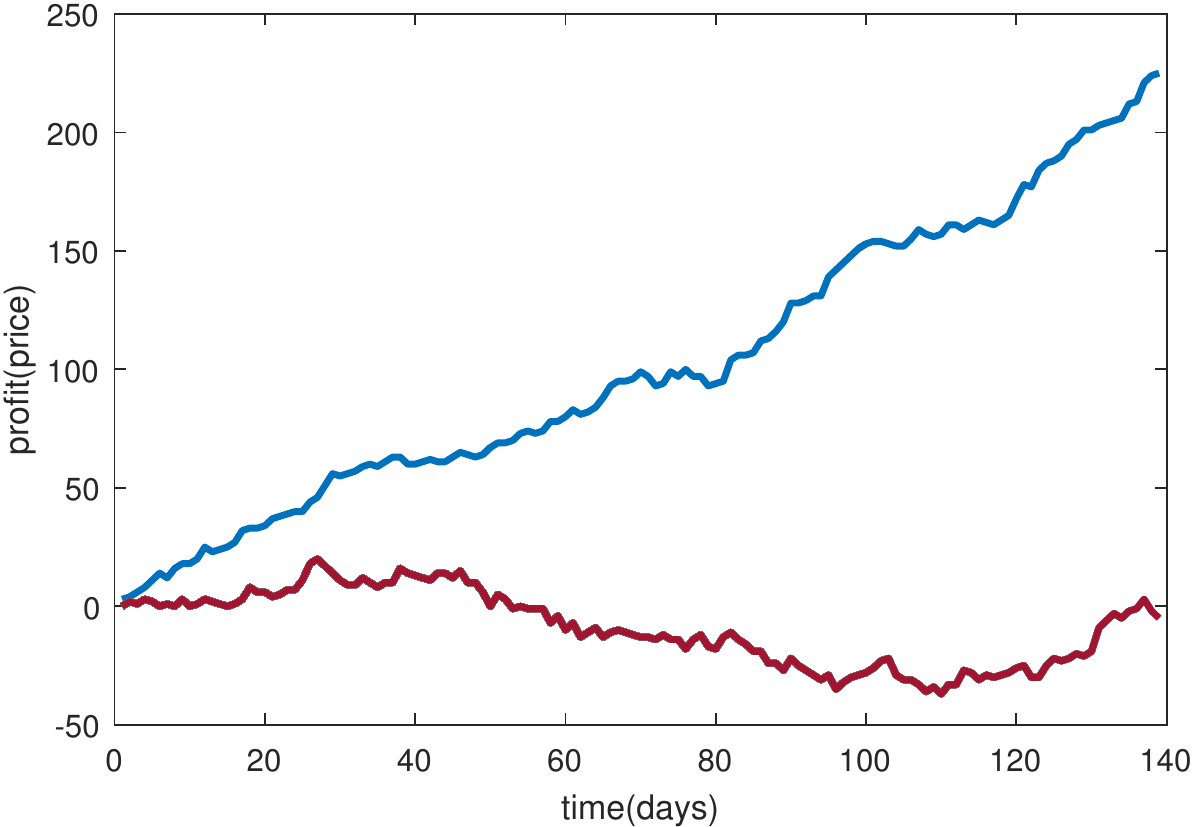}
	\caption{The profit of one day ahead directional prediction for the first SET50 stock, namely ADVANCE. SSM is a linear kernel of closed price with $\ICHAIN$ in blue line and without $\ICHAIN$ in red line. The data are run out of sample in directional prediction one day ahead within 138 days and trained with data from 23/08/2007 to 29/10/2015 with 2000 trading days.\label{result_single_stock}} 
\end{figure}

\subsection{Result of Empirical Analysis of SSM and Holo Hilbert Spectrum}

At first we compare the result of SSM with traditional SVM method. From the studies \cite{index2} follows that all stock indices around the world have similar statistical property, so called stylized fact \cite{stylized} of long memory process. The performance test is therefor done on six different financial time series data of stock market closed prices, see Table~\ref{input3}, which are the sample data for an analysis of country indices \cite{index} with SSM since all countries have the same rule for stock trading and the behaviors \cite{behavior} of traders among all countries are quite similar behaviors. In order to see the different result of performance between SVM and SSM we set the number of sample size to the same size as 31 days. The result of computation shows that SSM gets better performance than SVM, see Table~\ref{input}. The average performance of one day ahead of directional prediction for out of sample test of SVM is just $46.70\%$, but for SSM we get $57.38\%$. The improved performance from SVM to SSM is approximately $25.33\%$ for 31 days out of sample time series data. Next we can compare the performance of HT-SVM and HT-SSM methodology, see the results in Tables \ref{input2}, \ref{input4}. The HT-SVM is a Hilbert transform of $\ICHAIN$ followed by SVM procedure. Similarly, the HT-SSM means a Hilbert transform of $\ICHAIN$ folowed by SSM. The result of performance is tested again on six different types of time series data of stock closed prices. One can use different kernels and the same or longer date ranges than in the previous test. We confirm again that HT-SSM gets a better performance than HT-SVM.  We test the performance of directional prediction out of sample in 665 days of SET index. For interest, on our hardware configuration, it takes up to $3194.05$ seconds ($53.23$ min.) for one input sample point to train HT-SSM algorithm. Therefore within 665 days out of sample, we used computational time more than a month to show the result of analysis of SET index in this work. For SVM, the running time takes only $93.93$ seconds. The results of profit are shown in Fig.~\ref{result_thailand}, we can conclude that HT-SVM shows better performance than traditional SVM (also for the case of HT-SSM and SSM).

Next we test the different types of financial time series data. We choose a sample set of closed prices of 42 stocks underlying SET50 Index Futures as shown in Table~\ref{42stock}. The result of profit after directional prediction of SSM compared with SVM within 138 trading days are shown in Fig.~\ref{result_42stock}. The result for a profit of the first stock, namely ADVANC, is shown in Fig.~\ref{result_single_stock}. According to the empirical analysis of a sample set of closed prices, we found that for this particular data set of $42$ stocks, SSM shows better average performance in directional prediction of time series than SVM. The average performance of SSM in our empirical data set is 44.21\% in comparison with SVM which has only 38.55\%. The average profit of SSM $1234$ is also higher than SVM one with $158.91$ points. The general statements about SSM performance require detailed studies of more stock markets. Such studies are beyond the scope of this paper, however, the received results constitute a very useful ingredient for them. 

We check the performance of SSM with Hilbert transform of $\ICHAIN$ also with russian stock market RTS, which also confirm the conclusion that HT-SSM outperforms SVM method. The results are presented in Fig.~\ref{result_russia} with the performance of HT-SSM $54.93\%$ with the profit price $299.88$ in 142 trading days.

Finally, we implement the algorithm of SSM over Holo-Hilbert analysis of AM modulation of $\ICHAIN$ with SET index. The result of prediction is activated when time series data are in nonstationary state. It will generate signal of time series prediction two days ahead in directional prediction up or down with respect to the previous days. The empirical analysis with performance measuring is reported in Table~\ref{result_holo}.  The result of 100 days of 2nd layer Holo-Hilbert spectral FM mode of SET index and a performance test of comparison with HT-SSM are shown in Fig.~\ref{result}.

\begingroup
\squeezetable
\begin{table*}[!thb]
\renewcommand{\arraystretch}{1.3}
\centering
\begin{tabular}{lclllrll}
\toprule
\textbf{Date number} & \textbf{Number} & \textbf{Closed price} & \textbf{LH(2)-AM-SSM} & \textbf{Result} & \textbf{Profit} & \textbf{SM(rbs, $\mathbf{0.01}$)} & \textbf{State} \\
\colrule
9529&1&1276.84& up(forecast 9531)&correct&14.99&up(correct,+3.41)&\\ 
9530&2&1280.25& up(forecast 9532)&correct& 1.25 &up(correct,+14.99)&\\ 
9531&3&1295.24&up(forecast 9533)&wrong& -5.83&up(correct,+1.25)    &\\
9532&4&1296.49&up(forecast 9534)&correct&5.59&up(wrong,-5.83)&\\
9533&5&1290.66&up(forecast 9535)&correct&  17.81& up(correct,+5.59)&entanglement, [s4]\\
9534&6&1296.25&down(forecast 9536)&correct& 2.19&up(correct,+17.81)&\\
9535&7&1314.06&up(forecast 9537)&correct&  20.50&up(wrong,-2.19)&\\
9536&8& 1311.87&down(forecast 9538)&wrong&  -6.16 & up(correct,+20.50)&\\
9537&9& 1332.37&down(forecast 9539)&correct&  5.21&up(wrong,-6.16)&\\
9538&10& 1326.21&up(forecast 9540)&wrong&  -17.02&up(wrong, -5.21)    &normal, [s3]\\
9539&11&  1321.00&up(forecast 9541)&correct&0.23  & up(wrong,-17.02)&\\
9540&12&   1303.98&up(forecast 9542)&wrong&  -2.83&up(correct, 0.23)&\\
9541&13&  1304.21 &up(forecast 9543)&correct&2.50 &up(wrong, -2.83)&\\
9542&--&  1301.38 &no signal&-- &&&entanglement, [s2]\\
9543&14&1303.88&up(forecast 9545)&correct& 13.43& up(correct, 0.74)&\\
9544&15&1304.62&up(forecast 9546)&correct& 7.28& up(correct, 13.43)&\\
9545&16&1318.05&up(forecast 9547)& correct&13.88&up(correct,7.28)&\\
9546&17&1325.33 &up (forecast 9548)&correct&6.61&  up(correct,13.88)&\\
9547&18&1339.21 &up(forecast 9549)&correct&5.82&up(correct, 6.61) &\\
9548&19&1345.82 &down(forecast 9550)&wrong&-0.57&up(correct,5.82) &entanglement, s[1]\\
9549(--9558)&--&1351.64   &no signal&-- &&&\\
\colrule
Performance&&14/19&&71.68\% &84.88&\multicolumn{2}{c}{68.42\% (71.81), 13/19}\\
\colrule
9571&20&1392.01 &up(forecast 9573)&correct&2.68&up(wrong,-12.67)&\\ 
9572&--&1379.34&no signal&--&&&\\
9573&--&1382.02&no signal &--&&&\\
9574&21&1389.56& up(forecast 9576)&correct&12.68&up(wrong,-0.40)&\\
9575&22& 1389.16& up(forecast 9577)&correct&6.94 &up(correct,12.68)&\\
9576&--& 1401.84& no signal&--&&&\\
9577&23& 1408.78& up(forecast 9579)&correct&3.89 &up(correct,0.40)&\\
9578&24&  1409.18&up(forecast 9580) &correct&1.98& up(correct,  3.89)&\\
9579&25&  1413.07& up(forecast 9581)&correct&8.34& up(correct, 1.98)&\\
\colrule
Performance&&20/25&& 80\%&121.39&\multicolumn{2}{c}{68\% (77.69), 17/25}\\
\colrule
9580&26&  1415.05   & up(forecast 9582) &wrong&-0.72&up(correct,8.34)&\\
9581&--&  1423.39 & no signal&--&&&\\
9582&--&  1422.67 & no signal&--&&&entanglement, [s3]\\
9583&27&  1408.16 & up(forecast 9585) &correct&1.10& up(wrong,-14.51) &\\
9584(-9589)&& 1411.23   & no signal&--&&&\\
9590&28& 1379.02  &  up(forecast 9592) &wrong& -2.23&up(wrong,-1.65)&entanglement, [s2]\\
9591&29&  1377.37 & down(forecast 9593)&wrong& -20.89&up(wrong,-2.23)  &\\
9592&--& 1375.14  & no signal&--&&&entanglement, [s4]\\
9593&--& 1396.03  & no signal&--&&&\\
9594&30& 1395.21  & up(forecast 9596)&correct&5.37&up(correct,+10.05) &\\
9595&--&  1405.26 & no signal &--&&&\\
9596(-9602)&--&  1410.63 & no signal&--&&&\\
9603&31&  1402.79  &  up(forecast 9605)&correct&7.22&up(correct,5.72)&\\
\colrule
Performance&&23/31&& 74.19\%&111.24 &\multicolumn{2}{c}{64.51\% (83.41), 20/36}\\
\colrule
9604(--9606)&--&  1408.51  & no signal&&&&\\
9607&32&  1415.73  & up(forecast 9609)&correct&  3.76&up(wrong,-4.84)&\\
9608(--9610)& &   1449.40 & no signal&&&&\\
9611&33& 1457.30  & down(forecast 9613)&correct& 5.48&up(correct,11.89)&\\
9612&&  1469.19  & no signal&&&&\\
9613&& 1463.71 & no signal&&&&\\
9614&34& 1457.02 &  up(forecast 9616)&correct& 6.69&up(wrong,-1.00)&\\
9615&& 1456.02 & no signal&&&&\\
9616&35& 1471.85 & down(forecast 9618)  & correct& 19.66&up(wrong,-0.83)&\\
9617&& 1471.02  & no signal&&&&\\
9618&36& 1451.36 &  up (forecast 9620)& correct&5.38& up(correct,10.55)&\\
\colrule
Performance&&28/36&& 77.78\%&152& \multicolumn{2}{c}{60.00\% (172.45), 22/36}\\
\botrule
\end{tabular}
\caption{Table shows the performance for two day ahead directional prediction by using AM-mode over 2nd layer of Holo-Hilbert transform of $\ICHAIN$ prediction method of SSM. \label{result_holo}}
\end{table*}
\endgroup

\section{Discussion and Conclusion}\label{sec:conclusions}

\begin{figure}[!t]
	\centering
	\includegraphics[width=\linewidth]{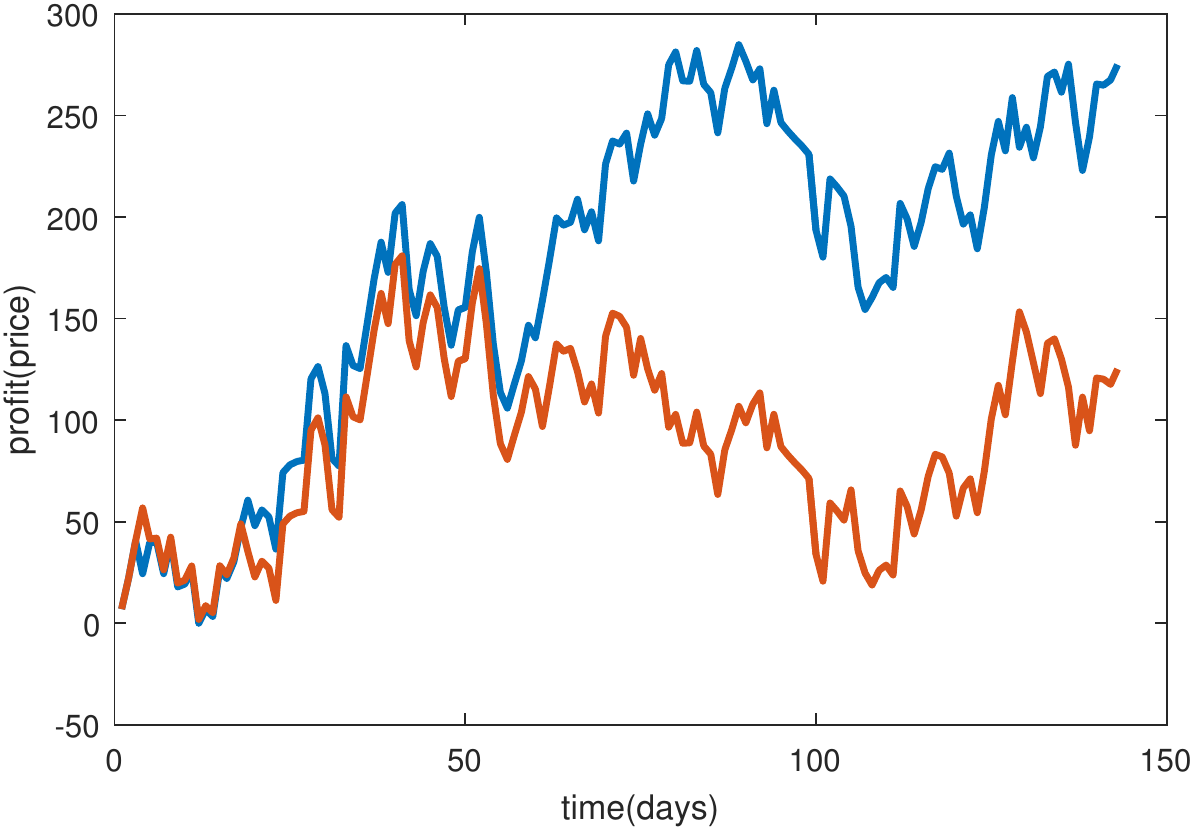}
	\caption{The profit of one day ahead directional prediction over out of sample 142 data set of closed price between 03/12/2015 to 01/07/2016 for Russian stock market. SSM with Hilbert transform HT-SSM (blue) and SVM with Hilbert transform HT-SVM (red) are with rbs kernel parameter set to $0.05$. The performance of HT-SSM is $54.93\%$ with profit price $299.88$ in 142 trading days. \label{result_russia}}
\end{figure}

We extend support vector machine to support spinor machine by using mathematical structure of spinor field and tensor product over vecor space. We define a support spinor from second cohomology group of moduli state space model in time series data.

We demonstrate the simple calculation of time series classification of SSM over $T_{0}$-separation axiom. We define a tensor field in financial time series data and obtain a spinor field from coupling states between them.
We finally prove that the separate hyperplane of SSM is separated by using degree with $\frac{\pi}{4}$.
We propose the equation of connection over Killing vector field in time series data. The solution of equation so called support Dirac machine is a weight of SSM. The equivalent class of weight of SSM can be written by integral over second cohomology group in time series data.
We prove the existence of $T_0$-separation criterion for SSM, different from $T_{2}$-separation of SVM in which two similar points of time series data can not be separated.

We report the empirical analysis performed with SSM and SVM, which confirms the better performance of SSM than SVM over six testing time series data from six different countries. We implement an algorithm of SSM over Holo-Hilbert analysis of AM. The result of extradimension of amplitude modulation confirms the existence of moduli state space model of time series data. We implement algorithm of SSM by using Holo-Hilbert amplitude modualtion mode for fully nonlinear and nonstationary time series data analysis. The result of analysis of hidden second layer can be used to predict market state of one day ahead with very high accuracy in comparison to typical SVM classification over state of end point in time series data. We compare the AM mode with Holo-Hilbert strategy, the AM mode strategy shows better results than HT-SVM and HT-SSM. It is also used to make two days ahead prediction with $71.68\%$ performance. The HT-SSM is just about $57\%$ and SVM is about $46.70\%$. The empirical analysis is done with out of sample test for all three strategies.

The SSM method promises a better performance than SVM. Although SVM is based on linear and nonliear classifier with stationary signal processing data, there is no evidence that SVM can be used for nonstationary signal processing data efficiently or not. The nonlinear and nonstationary signal processing data have more geometric general property than linear and stationary signal processing data for which SVM is widely use. There are many types of nonlinear and nonstationary signal processing data and financial time series data are only one type of them. Another types are data measured from stress tensor and tensor based visualizing image processing data from Video Share (VDO) streaming, by nature based on mathematical structure of tensor field and spinor field and thereafter more suitable for SSM than SVM. The algorithms based on SSM can replace SVM in cybernetics and electronics, it can be used for programming the robots and autonomic systems and completely replace SVM and neural networks.

\begin{figure*}[!t]
	\centering
	\includegraphics[height=.32\linewidth]{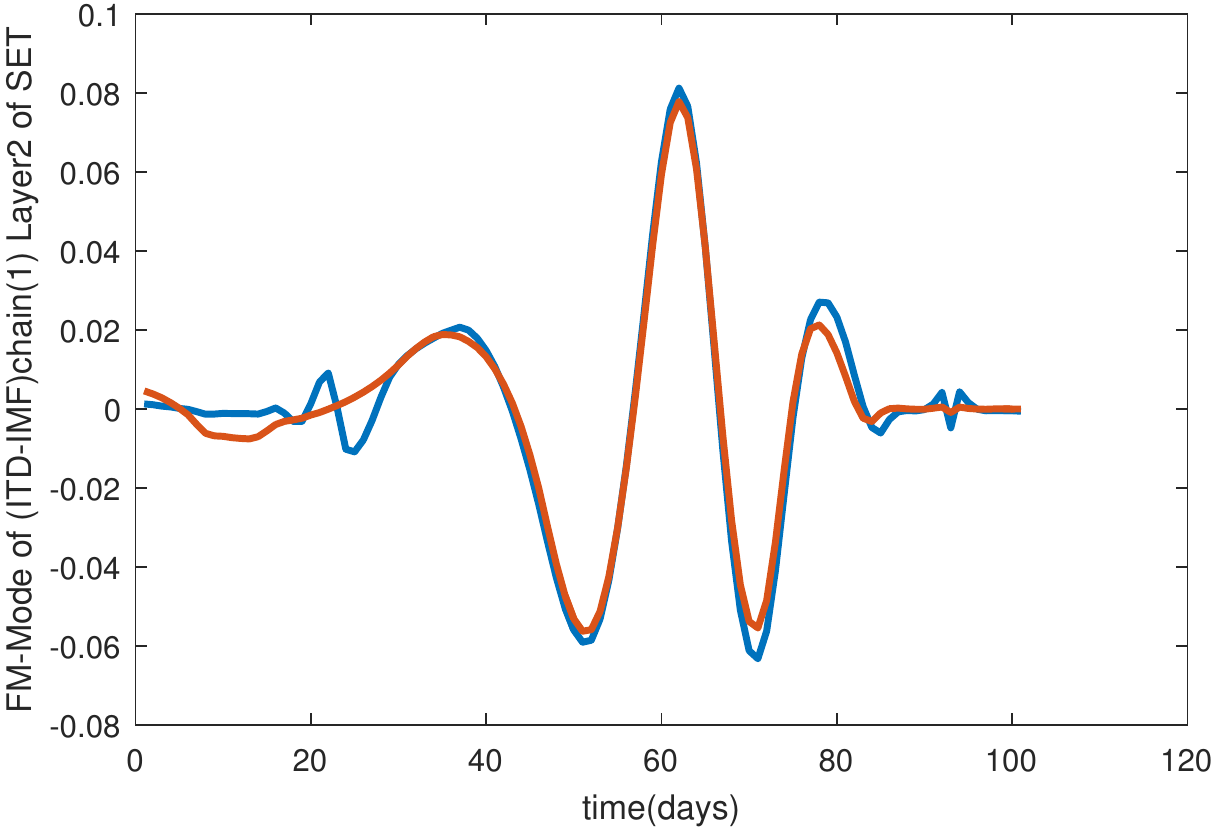}\hspace{2em}
	\includegraphics[height=.32\linewidth]{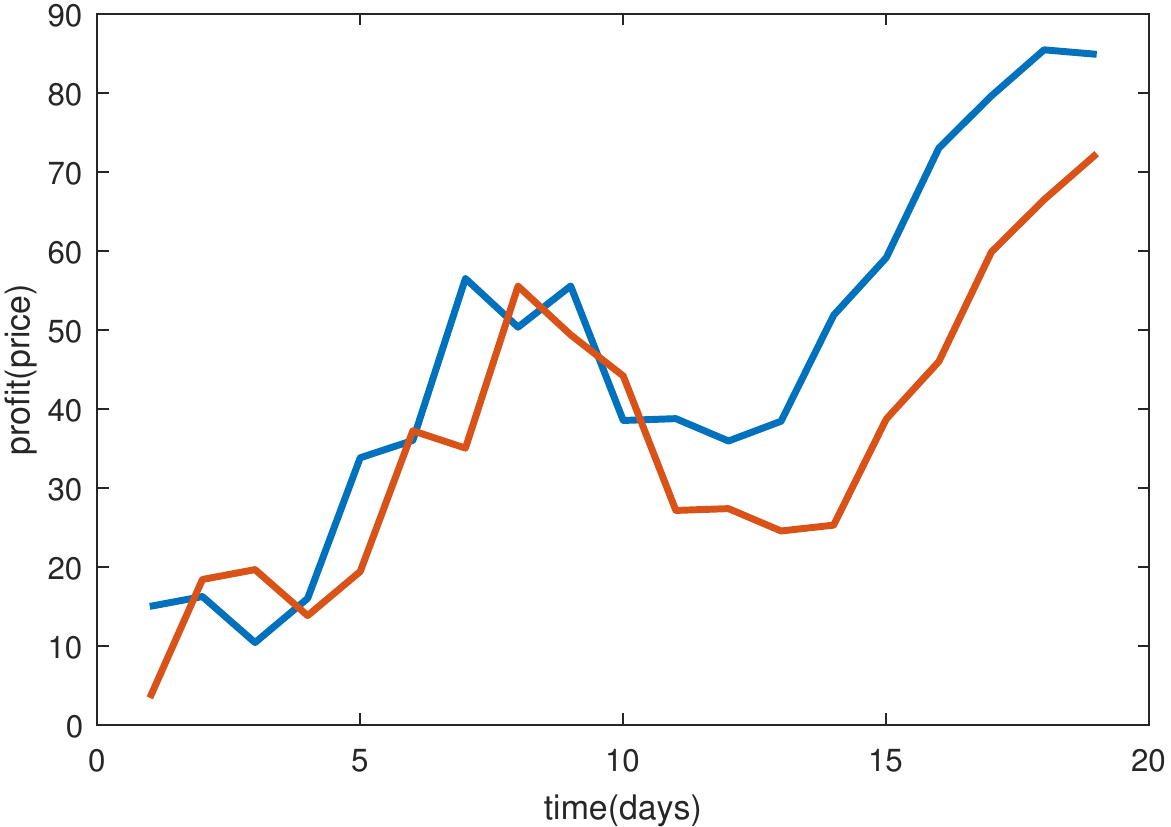}
	\caption{On the left: The FM mode of 2nd layer of $\ICHAIN$ for SET index in 9539 days (red) and in 9544 days (blue). The shape of frequency shows the stationary state of the last 100 days with normal state in up direction with two days ahead prediction of frequency. 
		On the right: The performance between two day ahead of Holo-Hilbert AM mode of 2nd layer with SSM price forecasting with highest profit $78.94\%$ within 19 days out of sample. The result of one day ahead of Holo-Hilbert AM mode of 2nd layer forecasting with highest profit also $78.94\%$ within 19 days out of sample. The HT-SSM has lower performance with $73.68\%$. All three strategies earn profit from the stock market. The highest profit is for two day ahead of Holo-Hilbert AM mode of 2nd layer with the value of $112$, HT-SSM can earn only $78.26$.\label{result}}
\end{figure*}

\section*{Acknowledgment}
K. Kanjamapornkul is supported by the scholarship from the 100th Anniversary Chulalongkorn University Fund for Doctoral Scholarship. The work was supported by Dean of Faculty of Engineering under the collaboration between Chulalongkorn University and SAS. The work was partly supported by the grants VEGA 2/009/16, 2/0153/17 and APVV-0463-12. R. Pinčák would like to thank the TH division in CERN for hospitality.


\end{document}